\PassOptionsToPackage{hyperfootnotes=false}{hyperref}
\documentclass[twoside,11pt]{article}

\usepackage[preprint]{jmlr2e}
\usepackage{amsmath}
\usepackage{amssymb}
\usepackage{bm}
\usepackage{booktabs}
\usepackage{lastpage}
\usepackage{mathtools}
\usepackage{microtype}
\usepackage{placeins}
\usepackage{tabularx}
\usepackage{tikz}
\usetikzlibrary{arrows.meta,positioning}

\hypersetup{
  hidelinks,
  pdftitle={Regularized Emphatic Temporal-Difference Learning: Stability under Constant Stepsizes},
  pdfauthor={Xingguo Chen et al.}
}

\newtheorem{assumption}{Assumption}
\newcommand{\E}{\mathbb{E}}
\newcommand{\Prob}{\mathbb{P}}

\newcommand{\trans}{\mathsf{T}}

\DeclareMathOperator{\Sym}{Sym}
\DeclareMathOperator{\diag}{diag}

\DeclareMathOperator{\spec}{spec}

\jmlrheading{}{2026}{1-\pageref{LastPage}}{}{}{}{
Xingguo Chen, Zhaohui Wu, Jinguo Ye, Chao Li, Shangdong Yang,
Guang Yang, Skylar Liang, and Wenhao Wang}
\ShortHeadings{Regularized Emphatic Temporal-Difference Learning}{Chen et al.}
\firstpageno{1}

\begin{document}

\title{Regularized Emphatic Temporal-Difference Learning: Stability under Constant Stepsizes}

\author{
\name Xingguo Chen\textsuperscript{1}\thanks{Corresponding author: chenxg@njupt.edu.cn.},
Zhaohui Wu\textsuperscript{1}, Jinguo Ye\textsuperscript{1}, Chao Li\textsuperscript{1}\\
\name Shangdong Yang\textsuperscript{1}\\
\addr \textsuperscript{1}Jiangsu Key Laboratory of Big Data Security \& Intelligent Processing,\\
Nanjing University of Posts \& Telecommunications, Nanjing 210023, China\\
\email chenxg@njupt.edu.cn
\AND
\name Guang Yang\textsuperscript{2}\\
\addr \textsuperscript{2}Nanjing University, 163 Xianlin Avenue, Nanjing 210023, China
\AND
\name Skylar Liang\textsuperscript{3}\\
\addr \textsuperscript{3}Microsoft Corporation, One Microsoft Way, Redmond, WA 98052, USA\\
\email liangtianyu@microsoft.com
\AND
\name Wenhao Wang\textsuperscript{4}\\
\addr \textsuperscript{4}College of Electronic Countermeasure, National University of Defense Technology,\\
Hefei 230037, China}

\maketitle

\begin{abstract}%
Emphatic temporal-difference learning (ETD) stabilizes the expected off-policy
TD update and changes its projection geometry, but neither property determines
constant-stepsize sampled dynamics. We construct an ergodic two-state
counterexample in which the ETD mean map contracts while the sampled product
has a positive top Lyapunov exponent. Regenerative-cycle analysis separates
this sign from the infinite variance of the follow-on trace. We introduce
regularized emphatic TD (RETD), a normalized first-order post-shock repair that
leaves the trace and importance ratios unchanged, stores the emphatic TD signal
in a leaky scalar state, and releases a delayed correction. RETD's raw
equilibrium is an affine shift of the ETD equilibrium; single- and
two-regularization readouts recover the ETD fixed point exactly. We prove
almost-sure convergence for harmonic diminishing stepsizes and a conditional
constant-stepsize moment-contraction result from a Markovian random-product
bound. RETD has certified negative exponents on the two-state construction and
one Baird point, whereas the positive Baird ETD sign remains numerical. Paired
10,000-run experiments validate both separations, fixed-point recovery, a
nonmonotone stability region, and task dependence. RETD changes post-shock
dynamics; it does not reduce the shared follow-on-trace variance.
\end{abstract}

\begin{keywords}
  off-policy learning, emphatic temporal-difference learning, constant stepsize,
  random matrix products, fixed-point recovery
\end{keywords}

\section{Introduction}
\label{sec:introduction}

Off-policy temporal-difference learning separates at least two questions:
whether the expected recursion is stable and whether its fixed point is a good
target-policy approximation.  Baird's counterexample showed that ordinary
off-policy TD can fail at the mean-update level
\citep{baird1995residual}; Gradient-TD methods instead stabilize projected
objectives \citep{sutton2008convergent,sutton2009fast}.  Oblique-projection
analyses further show that convergence need not imply a good approximation
\citep{scherrer2010should,kolter2011fixed}.

ETD addresses both issues through emphatic state weighting: its stationary
expected update is positive stable and its projected Bellman equation uses
emphatic weights \citep{sutton2016emphatic,hallak2016generalized}.  Diminishing-
stepsize convergence and constrained constant-step weak-convergence results
are known \citep{yu2015convergence,yu2016weak}, but these results do not settle
the unconstrained sampled recursion at a fixed nonzero stepsize.  A
finite-mean, infinite-variance follow-on trace identifies a shock source, not
the sign of the associated random-matrix product.

We study the gap in a fully specified two-state example: the target action has
behavior probability $0.3$, the ratios are $10/3$ and zero, $\gamma=0.9$,
the scalar features are $1$ and $2$, and $\alpha=0.1$.  The ETD mean multiplier
is $0.684$, while regenerative analysis proves a strictly positive scalar top
Lyapunov exponent.  We call this the \emph{mean-to-path stability gap}.

The same construction identifies the repair mechanism.  A long target-action
run creates a large emphatic shock, but the following zero-ratio ETD update is
the identity.  RETD adds one leaky scalar state, so later samples can release a
delayed correction even when the current ratio is zero; the trace and ratios
themselves remain unchanged.

Section~\ref{sec:retd} derives RETD from seven design constraints.  Positive
leakage removes the $c=0$ conserved mode, while the equilibrium analysis gives
single-$c$ and two-$c$ recovery of the ETD prediction, including singular and
ill-conditioned cases.  We then prove the diminishing-stepsize result and
instantiate the constant-stepsize random-product theorem of
\citet{durmus2021stability}; explicit counterexample points are handled by
regenerative or projective certificates.

Our contributions are threefold.  \emph{(1)} We give an exact two-state
mean-to-path separation: ETD is strictly mean stable but has a positive top
Lyapunov exponent, while RETD on the same task and stepsize has a certified
negative exponent.  \emph{(2)} We derive RETD as a one-scalar feedback
modification of the post-shock dynamics and characterize its equilibrium,
including single-$c$ and two-$c$ recovery of the ETD fixed point and the
associated singularity and conditioning conditions.  \emph{(3)} We prove
diminishing- and conditional constant-stepsize stochastic stability results,
give a rigorous RETD certificate and numerical ETD diagnosis on Baird's
counterexample, and evaluate mechanism, recovery, stability regions,
baselines, cost, and adverse tasks.

The paper follows this chain: Section~\ref{sec:preliminaries} defines the
stability notions and proves the ETD half of the two-state separation;
Section~\ref{sec:retd} derives RETD and its recovery identities; and
Section~\ref{sec:theory} gives the two stochastic theorems.  Section~\ref{sec:baird}
extends the mean/product analysis to Baird's system, Section~\ref{sec:experiments}
reports the paired evaluations and adverse tasks, and the final sections give
the literature position and limitations.

\section{Preliminaries and a Constant-Stepsize ETD Counterexample}
\label{sec:preliminaries}

We begin with the ordinary and emphatic projected systems and distinguish
fixed-point quality from three notions of dynamical stability.  A two-state
counterexample then shows why mean stability of ETD does not settle the
constant-stepsize sample-path question.

\subsection{Off-Policy Linear Prediction}

Consider a finite Markov decision process with state space $\mathcal S$,
action space $\mathcal A$, target policy $\pi$, behavior policy $\mu$, and a
constant discount $\gamma\in[0,1)$.  The behavior chain is irreducible and
aperiodic with stationary distribution $\bm d_\mu$, and the target policy is
covered by the behavior policy.  At time $t$, define
$\rho_t=\pi(A_t\mid S_t)/\mu(A_t\mid S_t)$ and
$v_{\bm\theta}(s)=\bm\phi(s)^\trans\bm\theta$, where the rows of $\bm\Phi$
are the feature vectors.  The TD error is
$\delta_t=R_{t+1}+\gamma\bm\phi(S_{t+1})^\trans\bm\theta_t
-\bm\phi(S_t)^\trans\bm\theta_t$.
Let $\bm P_\pi$ and $\bm r_\pi$ denote the target-policy transition matrix
and expected one-step reward.  Then
$\bm v_\pi=(\bm I-\gamma\bm P_\pi)^{-1}\bm r_\pi$.

Ordinary importance-sampled TD(0) has stationary mean system
\begin{equation*}
 \bm A_\mu=\bm\Phi^\trans\bm D_\mu
            (\bm I-\gamma\bm P_\pi)\bm\Phi,
 \qquad
 \bm b_\mu=\bm\Phi^\trans\bm D_\mu\bm r_\pi,
\end{equation*}
where $\bm D_\mu=\diag(\bm d_\mu)$.  Baird's counterexample shows that
$\bm A_\mu$ need not be positive stable, so the expected recursion itself can
diverge \citep{baird1995residual}.

\subsection{The Fixed Point Is a Separate Question}

Suppose an algorithm converges to
$\bm\theta_\mu=\bm A_\mu^{-1}\bm b_\mu$.  Its prediction is an oblique
projection whose quality depends on the sampling and projection geometry;
convergence therefore does not imply closeness to $\bm v_\pi$
\citep{scherrer2010should,kolter2011fixed}.

Reachability and quality are therefore separate questions.  Gradient-TD
methods stabilize projected objectives, but do not automatically replace
behavior weighting by the emphatic weighting relevant to ETD.

ETD(0), with unit interest, uses the follow-on trace
$F_{t+1}=1+\gamma\rho_tF_t$ and update
$\bm\theta_{t+1}=\bm\theta_t+
\alpha_tF_t\rho_t\delta_t\bm\phi(S_t)$.
Its expected emphatic weights and linear system are
\begin{align*}
 \bm f&=(\bm I-\gamma\bm P_\pi^\trans)^{-1}\bm d_\mu,
 &\bm F&=\diag(\bm f),
 \\
 \bm A&=\bm\Phi^\trans\bm F
        (\bm I-\gamma\bm P_\pi)\bm\Phi,
 &\bm b&=\bm\Phi^\trans\bm F\bm r_\pi.
\end{align*}
Under standard positive-interest and feature conditions, $\bm A$ is positive
stable and $\bm\theta_E=\bm A^{-1}\bm b$ solves the emphatically weighted
projected Bellman equation \citep{sutton2016emphatic,hallak2016generalized}.
Thus a path-stability modification must state how this fixed-point information
is preserved.

\subsection{Three Notions of Stability}
\label{sec:three-stabilities}

The phrase ``ETD is stable'' is incomplete unless the asymptotic regime and
random object are specified.  We distinguish three notions.

\emph{Mean Stability.}
With the trace process in stationarity and the parameter frozen, the expected
ETD drift is $\dot{\bm\theta}=\bm b-\bm A\bm\theta$.  The ODE is stable when
$-\bm A$ is Hurwitz; for constant $\alpha$,
$\rho(\bm I-\alpha\bm A)<1$ is the corresponding mean-map diagnostic.  Neither
is an expectation of a sampled product.

\emph{Diminishing-Stepsize Stochastic Stability.}
For a Robbins--Monro schedule with $\alpha_t\downarrow0$, one asks whether the
actual iterates remain bounded and converge almost surely.  This does not
imply stability when $\alpha_t\equiv\alpha>0$.

\emph{Constant-Stepsize Stochastic Stability.}
For the constant schedule $\alpha_t\equiv\alpha>0$, subtract an equilibrium
and write the homogeneous error recursion as
$\bm e_{t+1}=\bm M_t(\alpha)\bm e_t$.  The relevant long-run object is the ordered random product
$\bm M_{t-1}\cdots\bm M_0$.  On the identifiable prediction space, define
the upper top Lyapunov exponent
\begin{equation}
 \lambda(\alpha)=\limsup_{t\to\infty}\frac1t
 \log\|\bm M_{t-1}(\alpha)\cdots\bm M_0(\alpha)\|.
 \label{eq:top-lyapunov}
\end{equation}
A negative exponent gives exponential forgetting of initial perturbations and,
for an affine recursion, generally convergence to a stationary causal process
rather than a deterministic point.  A positive exponent gives an expanding top
direction; in the scalar counterexample every nonzero initial error grows
exponentially.

These notions are not linked by a simple expectation identity.  In particular,
$\E[\bm M_t\bm e_t]\ne \E[\bm M_t]\E[\bm e_t]$ because $F_t$, the state, and
the current iterate are correlated.  Infinite
variance of $F_t$ warns that shocks are severe, but does not determine the sign
of Equation~\ref{eq:top-lyapunov}; that sign must be analyzed from the product.

The matrices $\bm A$ and $\bm G_c$ used later are stationary mean matrices,
whereas sampled factors generate ordered random products.  A spectral-radius
calculation for a mean matrix cannot replace a Lyapunov-exponent calculation
for its sampled product.

\subsection{A Two-State Counterexample}
\label{sec:two-state-task}

Let $\mathcal S=\{s_1,s_2\}$ and
$\mathcal A=\{a_0,a_\pi\}$.  With rows denoting the current state and columns
the next state, the action-conditioned transition matrices are
\begin{equation*}
 \bm P^{a_0}=\begin{bmatrix}1&0\\1&0\end{bmatrix},
 \qquad
 \bm P^{a_\pi}=\begin{bmatrix}0&1\\0&1\end{bmatrix}.
\end{equation*}
\begin{figure}[t]
\centering
\begin{tikzpicture}[>=Latex,
  state/.style={circle,draw,minimum size=16mm,align=center,font=\small},
  target/.style={->,blue!65!black,semithick},
  reset/.style={->,gray!80,dashed,semithick}]
  \node[state] (s1) at (-2.1,0) {$s_1$\\$\phi(s_1)=1$};
  \node[state] (s2) at (2.1,0) {$s_2$\\$\phi(s_2)=2$};
  \draw[target] (s1) to[bend left=18]
    node[above,font=\small] {$a_\pi$} (s2);
  \draw[target] (s2) edge[loop right,looseness=5]
    node[right,font=\small] {$a_\pi$} (s2);
  \draw[reset] (s2) to[bend left=18]
    node[below,font=\small] {$a_0$} (s1);
  \draw[reset] (s1) edge[loop left,looseness=5]
    node[left,font=\small] {$a_0$} (s1);
  \node[align=center,font=\small] at (0,-1.55)
    {$\pi(a_\pi\mid s)=1$;\quad
     $\mu(a_\pi\mid s)=.3$, $\mu(a_0\mid s)=.7$\\
     $R_{t+1}=0$,\quad $\gamma=.9$,\quad $\alpha=.1$};
\end{tikzpicture}
\caption{The new two-state counterexample.  Action $a_\pi$ sends either state
to $s_2$, whereas $a_0$ sends either state to $s_1$.  The target policy always
chooses $a_\pi$ and the behavior policy chooses $a_\pi$ with probability $.3$.
Solid blue and dashed gray arrows distinguish the two actions and remain
legible in grayscale.}
\label{fig:two-state-counterexample}
\end{figure}
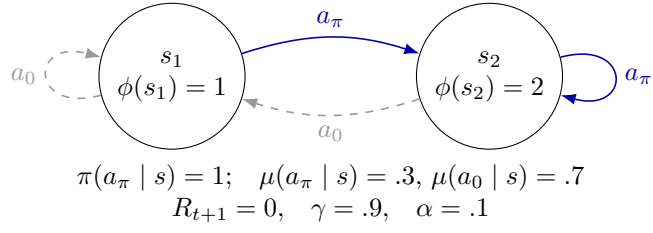
The target policy always selects $a_\pi$.  In either state the behavior policy
selects $a_\pi$ with probability $p=3/10$ and $a_0$ with probability $7/10$.
Therefore
\begin{equation*}
 \bm P_\mu=
 \begin{bmatrix}7/10&3/10\\7/10&3/10\end{bmatrix},
 \qquad
 \bm d_\mu=\begin{bmatrix}7/10\\3/10\end{bmatrix},
\end{equation*}
and $\rho(a_\pi)=10/3$, $\rho(a_0)=0$.  We set $R_{t+1}=0$,
$\gamma=9/10$, $\alpha=1/10$, $\phi(s_1)=1$, and $\phi(s_2)=2$.

The two states remain distinct: they have different transition destinations
under the two actions and different feature values.  A scalar parameter merely
restricts the two predictions to $(\theta,2\theta)$; it does not merge the
states.  Zero reward makes the true value $\bm v_\pi=\bm0$, which is exactly
represented by $\theta^*=0$.  The affine term then vanishes, so the homogeneous
stability question can be studied without approximation error.

\subsection{The Trace Has a Finite Mean and Infinite Variance}
\label{sec:two-state-trace-moments}

Let $K_t$ be the number of consecutive preceding target actions.  Independence
of behavior actions gives, in stationarity,
$\Prob(K_t=k)=(7/10)(3/10)^k$, $k=0,1,\ldots$.
After an $a_0$ action the next trace is one.  Each target action multiplies the
old trace by
$\gamma\rho=(9/10)(10/3)=3$ before adding one.  Hence
$F_k=1+3+\cdots+3^k=(3^{k+1}-1)/2$.  For every $r>0$, comparison with $3^k$
gives $\E[F_t^r]<\infty$ if and only if $(3/10)3^r<1$.  Since
$(3/10)3=0.9<1$ but $(3/10)3^2=2.7>1$,
so $\E[F_t]<\infty$ but $\E[F_t^2]=\infty$.  The first-moment recursion closes
because $\E[\gamma\rho_t]=\gamma<1$, whereas the second-moment coefficient is
$\E[(\gamma\rho_t)^2]=\gamma^2/p=2.7>1$.  This identifies the heavy-tailed
shock channel but does not prove product expansion.

\subsection{The ETD Mean Recursion Is Strictly Stable}

Here $\bm P_\pi=\bm P^{a_\pi}$.  With $\bm\phi=(1,2)^\trans$, the emphatic
weight vector and scalar ETD coefficient are
\begin{align*}
 \bm f&=(\bm I-\gamma\bm P_\pi^\trans)^{-1}\bm d_\mu
 =\begin{bmatrix}7/10\\93/10\end{bmatrix},\\
 A
 &=\bm\phi^\trans\diag(\bm f)
    (\bm I-\gamma\bm P_\pi)\bm\phi \\
 &=1\cdot\frac7{10}\left(-\frac45\right)
   +2\cdot\frac{93}{10}\left(\frac15\right)
 =\frac{79}{25}.
\end{align*}
The deterministic mean error multiplier is
\begin{equation}
 1-\alpha A=1-\frac1{10}\frac{79}{25}
 =\frac{171}{250}=0.684.
 \label{eq:two-state-mean-factor-new}
\end{equation}
Thus the mean recursion has a strict contraction margin.

\subsection{The Constant-Stepsize ETD Product Nevertheless Expands}
\label{sec:two-state-etd-product}

Use cycles of $L$ target actions followed by one $a_0$ action.  Then
$\Prob(L=\ell)=(7/10)(3/10)^\ell$ and $\E[L+1]=10/7$.
The first target action moves from $s_1$ to $s_2$; every later target action
remains in $s_2$.  Their scalar error multipliers are
\begin{align}
 m_0&=1+\alpha F_0\frac{10}{3}
       (\gamma\phi(s_2)-\phi(s_1))\phi(s_1)=\frac{19}{15},
 \label{eq:two-state-etd-m0-new}\\
 m_k&=1+\alpha F_k\frac{10}{3}
       (\gamma\phi(s_2)-\phi(s_2))\phi(s_2)
 =1-\frac{2F_k}{15}=\frac{16-3^{k+1}}{15},\quad k\ge1.
 \label{eq:two-state-etd-mk-new}
\end{align}
The final $a_0$ action has $\rho=0$, so ETD applies the identity.  The cycle
factor is $C_L=\prod_{k=0}^{L-1}m_k$, with $C_0=1$.  Its first four nonempty
values are $C_1=19/15$, $C_2=133/225$ and $C_3=-1463/3375$,
$C_4=19019/10125$.
Direct outward-rounded evaluation gives
\begin{equation*}
 \sum_{L=1}^{4}\frac7{10}\left(\frac3{10}\right)^L
 \log|C_L|>0.004295228922459683.
\end{equation*}
Moreover $|C_4|>1$ and all factors appended after $m_3$ have magnitude above
one, so omitted terms with $L\ge5$ are nonnegative.

For the full-time limit, there are constants $a_0,a_1<\infty$ such that
$|\log|m_k||\le a_0+a_1k$ for $k\ge0$, because none of the factors in
Equations~\ref{eq:two-state-etd-m0-new}--\ref{eq:two-state-etd-mk-new} vanishes
and their logarithms grow at most linearly in $k$.  Consequently,
\begin{equation}
 |\log|C_L||\le a_0L+\frac{a_1}{2}L(L-1).
 \label{eq:two-state-cycle-log-integrable-new}
\end{equation}
The geometric law of $L$ therefore gives $\E|\log|C_L||<\infty$.

Let $L_1,L_2,\ldots$ be the successive iid run lengths, put
$\tau_n=L_n+1$, and let $T_N=\sum_{n=1}^N\tau_n$ be the $N$th cycle endpoint.
For every nonzero initial error, the strong law gives
\begin{equation*}
 \frac{1}{T_N}\log\left|\frac{\theta_{T_N}}{\theta_0}\right|
 =\frac{N}{T_N}\frac1N\sum_{n=1}^N\log|C_{L_n}|
 \longrightarrow
 \frac{\E\log|C_L|}{\E[L+1]}
 \quad\text{a.s.}
\end{equation*}
The endpoint limit is also the full-time limit.  If
$T_N\le t<T_{N+1}$, the unfinished part is bounded by
Equation~\ref{eq:two-state-cycle-log-integrable-new} with $L=L_{N+1}$.  For
every $\varepsilon>0$,
$\sum_{n=1}^{\infty}\Prob(L_n^2>\varepsilon n)<\infty$; the geometric tail
and Borel--Cantelli give $L_n^2/n\to0$ almost surely, while
$T_N/N\to\E[L+1]>0$.  The unfinished-cycle contribution is therefore $o(t)$,
and the per-transition exponent is
\begin{equation}
 \lambda_E=
 \frac{\E\log|C_L|}{\E[L+1]}
 >0.003006660245721778>0.
 \label{eq:two-state-etd-positive-new}
\end{equation}

Equations~\ref{eq:two-state-mean-factor-new} and
\ref{eq:two-state-etd-positive-new} give the separation: the mean multiplier
is $171/250$, yet every nonzero scalar ETD error has almost-sure logarithmic
growth above $0.003006660245721778$.  The interpolation above makes this a
complete-trajectory statement, not only a cycle-endpoint statement.

The trace-moment calculation identifies the shock source, whereas the cycle
calculation proves product expansion.  The final $\rho=0$ identity cannot
respond to the preceding target-action run; Section~\ref{sec:retd} modifies
this post-shock response.

\section{RETD: Delayed Shock Recovery and Solution Preservation}
\label{sec:retd}

The two-state analysis in Section~\ref{sec:two-state-task} identifies two facts
that must be kept separate.  ETD's emphatic weighting gives a stable and useful
mean system, but the sampled constant-stepsize recursion can accumulate an
expanding ordered product.  The algorithm proposed here therefore changes the
response to an emphatic shock without discarding the emphatic trace or its
fixed-point information.

\subsection{Design Principles and the Unit-Loop-Gain Choice}

Write the scalar ETD signal as $g_t=F_t\rho_t\delta_t$.
After a large $g_t$, the next ratio may be zero, so ETD applies no correction to
the history summarized by $F_t$.  We require the modification to leave $F_t$
and $\rho_t$ unchanged, react after zero-ratio samples, add and decay only one
scalar state, preserve the solution along a recoverable direction, and retain
$O(d)$ cost.
These requirements are met by the following first-order scalar-feedback
family:
\begin{equation}
\begin{aligned}
 \bm\theta_{t+1}
 &=\bm\theta_t+\alpha_t(g_t-\kappa\omega_t)\bm\phi(S_t),\\
 \omega_{t+1}
 &=\omega_t+\alpha_t\{\eta g_t-\beta\omega_t\},
 \qquad \kappa,\eta,\beta>0.
\end{aligned}
\label{eq:general-first-order-repair}
\end{equation}
If $\widetilde\omega_t=\kappa\omega_t$, the updates become
$\bm\theta_{t+1}=\bm\theta_t+
\alpha_t(g_t-\widetilde\omega_t)\bm\phi(S_t)$ and
$\widetilde\omega_{t+1}=\widetilde\omega_t+
\alpha_t\{\ell g_t-\beta\widetilde\omega_t\}$, where only the loop-gain
product $\ell=\kappa\eta$ remains.  The coordinate change sets the coefficient
in the $\bm\theta$ update to one but leaves $\ell$ invariant.  RETD chooses
$\ell=1$,
so that the shock enters both channels with the same gain.  We use the
representative $\kappa=\eta=1$ and write $\beta=1+c$.  This defines the
one-parameter family; it is a design choice, not a uniqueness claim for the
broader family in Equation~\ref{eq:general-first-order-repair}.

The resulting RETD recursion is
\begin{equation}
\begin{aligned}
 \bm\theta_{t+1}
 &=\bm\theta_t+\alpha_t(g_t-\omega_t)\bm\phi(S_t),\\
 \omega_{t+1}
 &=\omega_t+\alpha_t\{g_t-(1+c)\omega_t\},
 \qquad c>0.
\end{aligned}
\label{eq:retd-update-new}
\end{equation}
The follow-on trace and $\rho_t$ are unchanged, so RETD does not make the
infinite trace variance finite.  It changes only the joint dynamics after the
same shock, with one inner product and one scalar recursion per update.

\subsection{Impulse Response and the Necessity of Positive Leakage}

For a constant stepsize, a sample with $\rho_t=0$ has homogeneous matrix
\begin{equation}
 \bm M_t^R=
 \begin{bmatrix}
  \bm I&-\alpha\bm\phi(S_t)\\
  \bm0^\trans&1-\alpha(1+c)
 \end{bmatrix}.
 \label{eq:retd-zero-ratio-new}
\end{equation}
Unlike ETD's identity update, this matrix retains and decays the stored shock
while correcting the value parameter.  The auxiliary channel requires
$0<\alpha(1+c)<2$, hence $c<2/\alpha-1$ for fixed $\alpha$.

\begin{proposition}[Impulse response of the recovery channel]
\label{prop:retd-impulse-new}
Fix $\alpha,c$ with $0<\alpha(1+c)<2$ and put
$q_c=1-\alpha(1+c)$.  If $\omega_\tau=0$, $g_\tau=g$, and
$g_{\tau+j}=0$ for $j\ge1$, then
\begin{align*}
 \omega_{\tau+j}&=\alpha gq_c^{j-1}, &&j\ge1,
 \\
 \bm\theta_{\tau+j+1}-\bm\theta_{\tau+j}
 &=-\alpha^2gq_c^{j-1}\bm\phi(S_{\tau+j}), &&j\ge1.
\end{align*}
If the post-shock features equal $\bar{\bm\phi}$, the total correction is
$-\alpha g\bar{\bm\phi}/(1+c)$.
\end{proposition}

\begin{proof}
At the shock $\omega_{\tau+1}=\alpha g$; thereafter
$\omega_{t+1}=q_c\omega_t$.  Substitution into the first update and summing
the geometric series (using $|q_c|<1$) gives both identities.
\end{proof}

A positive leakage removes a conserved mode at $c=0$.  If the constant
function lies in the feature span, choose $\bm\ell$ with
$\bm\ell^\trans\bm\phi(s)=1$.  In a zero-reward task,
Equation~\ref{eq:retd-update-new} gives
$J_t=\bm\ell^\trans\bm\theta_t-\omega_t$, which is preserved sample by sample.
The joint product therefore has a
noncontracting left direction.  For $c>0$,
$J_{t+1}=J_t+\alpha_tc\omega_t$, so strict contraction is no longer ruled out,
although some positive values of $c$ may still be unstable.

\subsection{Mean Equilibrium and the Quality of the RETD Solution}
\label{sec:retd-fixed-point}

Let $\bm q=\bm\Phi^\trans\bm d_\mu$ and
$\beta=\bm f^\trans\bm r_\pi$.
In an identifiable full-rank feature coordinate, the stationary RETD mean
system is
\begin{equation*}
 \bm G_c
 \begin{bmatrix}\bm\theta_c\\\omega_c\end{bmatrix}
 =\bm h,
 \qquad
 \bm G_c=
 \begin{bmatrix}\bm A&\bm q\\\bm q^\trans&1+c\end{bmatrix},
 \qquad
 \bm h=\begin{bmatrix}\bm b\\\beta\end{bmatrix}.
\end{equation*}
If the original $n\times d$ feature matrix has rank $r<d$, let
$\bm B\in\mathbb R^{d\times r}$ have orthonormal columns spanning
$\operatorname{row}(\bm\Phi)$ and use
\begin{equation}
 \widetilde{\bm\Phi}=\bm\Phi\bm B,
 \quad
 \widetilde{\bm A}=\bm B^\trans\bm A\bm B,
 \quad
 \widetilde{\bm q}=\bm B^\trans\bm q,
 \quad
 \widetilde{\bm b}=\bm B^\trans\bm b.
 \label{eq:identifiable-basis-new}
\end{equation}
Then $\widetilde{\bm\Phi}$ has full column rank and
$\widetilde{\bm\theta}=\bm B^\trans\bm\theta$ is identifiable.  All inverses
below are in these coordinates; the orthogonal parameter component is unchanged
and prediction-null.

Define $\bm u=\bm A^{-1}\bm q$,
$s=\bm q^\trans\bm A^{-1}\bm q$,
$m_E=\beta-\bm q^\trans\bm\theta_E$, and
$\bm\theta_E=\bm A^{-1}\bm b$.

\begin{proposition}[Exact single-$c$ recovery]
\label{prop:single-c-recovery-new}
If $1+c-s\ne0$, the RETD equilibrium is unique and satisfies
\begin{equation}
 \omega_c=\frac{m_E}{1+c-s},
 \qquad
 \bm\theta_c=\bm\theta_E-\bm A^{-1}\bm q\,\omega_c.
 \label{eq:retd-shift-new}
\end{equation}
Consequently,
\begin{equation*}
 \boxed{\bm\theta_E=\bm\theta_c+\bm A^{-1}\bm q\,\omega_c.}
\end{equation*}
\end{proposition}

\begin{proof}
Subtracting $\bm A\bm\theta_E=\bm b$ from the first block equation gives
$\bm\theta_c-\bm\theta_E=-\bm A^{-1}\bm q\omega_c$.  Substitution into the
second gives $(1+c-s)\omega_c=m_E$ and hence both formulas.  The same condition
is equivalent to $\det\bm G_c=\det(\bm A)(1+c-s)\ne0$.
\end{proof}

Equation~\ref{eq:retd-shift-new} shows that the raw parameter differs from ETD
only along $\bm u$; in value space,
$\bm\Phi\bm\theta_c=\bm\Phi\bm\theta_E-(\bm\Phi\bm u)\omega_c$.  The recovered
output cancels this direction and inherits the ETD fixed-point
guarantee under the same assumptions and norm.  This is an equilibrium identity,
not equality of finite-time trajectories, covariances, or action rankings.

If $1+c-s=0$ and $m_E\ne0$, the mean equations are inconsistent; if both are
zero, the equilibrium is not unique.  If $m_E=0$ and $1+c-s\ne0$, then
$\omega_c=0$ and raw RETD already equals ETD, as in the zero-reward examples.

\begin{corollary}[Two-$c$ model-free readout]
\label{cor:two-c-recovery-new}
Let $(\bm\theta_1,\omega_1)$ and $(\bm\theta_2,\omega_2)$ be RETD equilibria
for distinct regularizers.  If $\omega_1\ne\omega_2$, then
\begin{equation}
 \boxed{
 \bm\theta_E=
 \frac{\omega_1\bm\theta_2-\omega_2\bm\theta_1}
      {\omega_1-\omega_2}.}
 \label{eq:two-c-readout-new}
\end{equation}
Moreover,
$|\omega_1-\omega_2|=|m_E|\,|c_2-c_1|/
(|1+c_1-s|\,|1+c_2-s|)$.
\end{corollary}

\begin{proof}
Use $\bm\theta_i=\bm\theta_E-\bm u\omega_i$ from
Proposition~\ref{prop:single-c-recovery-new}; eliminating $\bm u$ gives the
readout, and subtracting $\omega_i=m_E/(1+c_i-s)$ gives its conditioning
formula.
\end{proof}

The two-$c$ formula is asymptotic, not finite-time ETD reproduction, and becomes
ill conditioned when $m_E$ is near zero, the regularizers are close, or
estimation errors approach $|\omega_1-\omega_2|$.  In the zero-reward examples
both auxiliary equilibria are zero, giving $0/0$ because no recovery is needed.

\emph{Which Output Is Used in Prediction.}
The raw online predictor is $\bm\Phi\bm\theta_t$.  If the model quantities
$\bm A$ and $\bm q$ are available, the single-$c$ predictor instead uses
$\bm\Phi\{\bm\theta_t+\bm A^{-1}\bm q\,\omega_t\}$.
This model-based readout does not alter learner dynamics.  If
$\bm A^{-1}\bm q$ is unavailable, two RETD learners with distinct $c$ values can
share a trajectory and apply Equation~\ref{eq:two-c-readout-new} after averaging;
this doubles the parameter updates and must report
$|\bar\omega_1-\bar\omega_2|$.  Section~\ref{sec:experiments} evaluates both
readouts on nonzero-reward tasks.

\subsection{RETD Closes the Two-State Sign Separation}
\label{sec:two-state-retd}

Return to the counterexample in Section~\ref{sec:two-state-task}.  Here
$q=\E[\phi(S_t)]=13/10$.
At $c=4$, the RETD mean matrix is
\begin{equation}
 \bm G_4=
 \begin{bmatrix}79/25&13/10\\13/10&5\end{bmatrix},
 \qquad
 \det(\bm G_4)=\frac{1411}{100}>0.
 \label{eq:two-state-retd-mean-new}
\end{equation}
It is symmetric positive definite, and
$\rho(\bm I-0.1\bm G_4)\approx0.7513<1$.  Thus both ETD and RETD have stable
mean recursions on the same counterexample.

For the regenerative product, put
$h_k=F_k\rho_k\{\gamma\phi(S_{k+1})-\phi(S_k)\}$.
At $c=4$, the auxiliary retention is
$1-\alpha(1+c)=1/2$.  The $k$th target-action matrix is
\begin{equation*}
 \bm T_k=
 \begin{bmatrix}
 1+\alpha h_k\phi_k&-\alpha\phi_k\\
 \alpha h_k&1/2
 \end{bmatrix},
\end{equation*}
where $h_0=8/3$, $\phi_0=1$, and
$h_k=-2F_k/3$, $\phi_k=2$ for $k\ge1$.
The final off-target action has matrix
\begin{equation*}
 \bm D_1=\begin{bmatrix}1&-1/10\\0&1/2\end{bmatrix},
 \qquad
 \bm D_2=\begin{bmatrix}1&-1/5\\0&1/2\end{bmatrix},
\end{equation*}
depending on the current state.  The complete cycle matrices are
\begin{equation*}
 \bm M_0=\bm D_1,
 \qquad
 \bm M_L=\bm D_2\bm T_{L-1}\cdots\bm T_0,
 \quad L\ge1.
\end{equation*}
The dashed/off-target matrix is essential: omitting it would remove the
recovery mechanism.

Use the induced similarity norm
\begin{equation*}
 \bm S=\begin{bmatrix}4/3&-1/3\\0&3/4\end{bmatrix},
 \qquad
 \|\bm M\|_S=\|\bm S\bm M\bm S^{-1}\|_2.
\end{equation*}
Exact rational cycle matrices and outward-rounded spectral-norm calculations
give the certificate in Table~\ref{tab:two-state-retd-certificate-new}.

\begin{table}[t]
\centering
\footnotesize
\setlength{\tabcolsep}{4pt}
\renewcommand{\arraystretch}{.95}
\begin{tabular}{c c r}
\toprule
$L$ & $\Prob(L)$ & upper bound on
$\Prob(L)\log\|\bm M_L\|_S$\\
\midrule
0 & $7/10$ & $\phantom{-}0.000919795298289812$\\
1 & $21/100$ & $\phantom{-}0.035635247962602710$\\
2 & $63/1000$ & $-0.029235245805596898$\\
3 & $189/10000$ & $-0.026797225126525312$\\
4 & $567/100000$ & $\phantom{-}0.002112080306350725$\\
5 & $1701/1000000$ & $\phantom{-}0.005207967408664208$\\
6 & $5103/10000000$ & $\phantom{-}0.003536032736317089$\\
\midrule
sum & & $-0.008621347219897670$\\
\bottomrule
\end{tabular}
\caption{Finite prefix of the two-state RETD cycle certificate.  The cycle
matrices are rational; only outward-rounded logarithmic bounds are printed.}
\label{tab:two-state-retd-certificate-new}
\end{table}

For every $L\ge7$, the inequalities
$\|\bm S\|_2\|\bm S^{-1}\|_2<20/9$,
$\|\bm D_2\|_2<11/10$, $\|\bm T_k\|_2\le(9/5)F_k$, and
$F_k<(3/2)3^k$ give
\begin{equation*}
 \log\|\bm M_L\|_S
 \le \log\frac{22}{9}+L\log\frac{27}{10}
      +\frac{L(L-1)}2\log3.
\end{equation*}
Multiplying by the geometric cycle probabilities and summing the tail yields
\[
 \sum_{L=7}^{\infty}\frac7{10}\left(\frac3{10}\right)^L
 \log\|\bm M_L\|_S<0.007619669942317829.
\]
Together with the finite prefix,
$\E\log\|\bm M_L\|_S<-0.0010016772775798414$.  Dividing by the mean cycle
length gives the cycle-endpoint bound.  For every transition, let
$\bm\Pi_N=\bm M_{L_N}\cdots\bm M_{L_1}$ and
$T_N=\sum_{n=1}^N(L_n+1)$.
Submultiplicativity and the strong law imply
\begin{equation}
 \limsup_{N\to\infty}\frac1{T_N}\log\|\bm\Pi_N\|_S
 \le
 \frac{\E\log\|\bm M_L\|_S}{\E[L+1]}
 \quad\text{a.s.}
 \label{eq:two-state-retd-boundary-new}
\end{equation}
The same factor bounds give constants $d_0,d_1<\infty$ such that every partial
product within a cycle of length $L$ satisfies
$\max_j\log^+\|\bm T_{j-1}\cdots\bm T_0\|_S\le d_0+d_1L^2$, where
$\log^+x=\max\{0,\log x\}$.
The geometric cycle law and Borel--Cantelli give $L_n^2/n\to0$, whereas
$T_N/N\to10/7$; the unfinished cycle is therefore $o(t)$.  Hence
Equation~\ref{eq:two-state-retd-boundary-new} holds for the complete product
and proves
\begin{equation}
 \lambda_R(\alpha=0.1,c=4)
 \le-0.0007011740943058889<0.
 \label{eq:two-state-retd-negative-new}
\end{equation}

Together, Equations~\ref{eq:two-state-etd-positive-new},
\ref{eq:two-state-retd-mean-new}, and
\ref{eq:two-state-retd-negative-new} give the sign separation.  Both mean maps
contract, but at $\alpha=.1$,
$\lambda_E>0.003006660245721778$ while
$\lambda_R\le-0.0007011740943058889$.
The $S$-norm bound and the unfinished-cycle estimate give almost-sure
contraction at every transition.  Because rewards are zero, RETD's affine
residual vanishes sample by sample.  The sign difference comes from the final
zero-ratio action: ETD applies the identity, whereas RETD's $\bm D_i$ releases
the delayed correction and decays the stored shock.  Section~\ref{sec:theory}
treats stochastic stability beyond this example.

\section{Stochastic Stability of RETD}
\label{sec:theory}

The expected RETD matrix supplies the common geometric certificate for two
stochastic regimes.  We first establish almost-sure convergence with a
diminishing stepsize and then establish contraction of the constant-stepsize
random product.  The first result is a stochastic-approximation statement;
the second addresses the mean-to-path question isolated by the counterexample.

\subsection{Common Geometry and the Explicit Role of \texorpdfstring{$c$}{c}}

Work in an identifiable, full-column-rank feature coordinate and let
$\bm H=\Sym(\bm A)=(\bm A+\bm A^\trans)/2$.

\begin{assumption}[Finite off-policy prediction model]
\label{ass:common-theory-new}
The state and action spaces are finite.  The behavior chain is irreducible and
aperiodic, the target policy is covered by the behavior policy, rewards are
bounded functions of $(S_t,A_t,S_{t+1})$, interest is one,
$\gamma\in[0,1)$ is constant, and $\lambda=0$.  The identifiable feature
matrix has full column rank, its rows are bounded, and
$\bm H\succ0$.
\end{assumption}

\begin{assumption}[RETD regularization]
\label{ass:c-theory-new}
The parameter $c$ is chosen before the run and is required to satisfy
$c>c_{\min}:=\max\{0,\bm q^\trans\bm H^{-1}\bm q-1\}$.
\end{assumption}

This is a direct hypothesis of both RETD theorems, rather than a choice made
inside either proof.  Under Assumption~\ref{ass:common-theory-new}, it is
equivalent to positive definiteness of the symmetric RETD mean matrix.  To see
this, define
\begin{equation*}
 \Delta_c=1+c-\bm q^\trans\bm H^{-1}\bm q>0,
 \qquad
 \bm U=\begin{bmatrix}\bm I&\bm H^{-1}\bm q\\\bm0^\trans&1\end{bmatrix}.
\end{equation*}
Then the exact congruence
\begin{equation}
 \Sym(\bm G_c)=
 \begin{bmatrix}\bm H&\bm q\\\bm q^\trans&1+c\end{bmatrix}
 =\bm U^\trans
 \begin{bmatrix}\bm H&\bm0\\\bm0^\trans&\Delta_c\end{bmatrix}\bm U
 \succ0
 \label{eq:retd-schur-congruence-new}
\end{equation}
shows that $-\bm G_c$ is Hurwitz.  In particular,
\begin{equation}
 \underline m_c=
 \frac{\min\{\lambda_{\min}(\bm H),\Delta_c\}}
      {\|\bm U^{-1}\|_2^2}
 \le\lambda_{\min}\{\Sym(\bm G_c)\}.
 \label{eq:mc-new}
\end{equation}
The bound follows from the RETD mean matrix and does not assume contraction of
the sampled matrix product.

\subsection{Diminishing-Stepsize Stochastic Stability}

The follow-on trace is unbounded, so an independent, bounded-noise
Robbins--Monro template is not sufficient.  We use the invariant-law and
pathwise averaging results for the ETD trace process in
\citet{yu2015convergence} and the general-state-space ODE theorem of
\citet{liu2025ode}.  Both are external results; the proof below verifies their
conditions for the actual RETD recursion.

\begin{theorem}[Diminishing-stepsize RETD]
\label{thm:retd-diminishing-new}
Suppose the common conditions in Assumption~\ref{ass:common-theory-new} and the
explicit $c$ condition in Assumption~\ref{ass:c-theory-new} hold.  Let
$\alpha_t=B_1/(t+B_2)$, where $B_1,B_2>0$.
Then the joint RETD iterates
$\bm z_t=[\bm\theta_t^\trans,\omega_t]^\trans$ are almost surely bounded and
converge to the unique RETD mean equilibrium:
\begin{equation}
 \sup_t\|\bm z_t\|_2<\infty\quad\text{a.s.},
 \qquad
 \bm z_t\longrightarrow \bm z_c^*=\bm G_c^{-1}\bm h
 \quad\text{a.s.}
 \label{eq:diminishing-conclusion-new}
\end{equation}
For a rank-deficient original representation, the conclusion holds in the
identifiable coordinate; the parameter-null component is unchanged and the
predictions converge to the corresponding RETD equilibrium prediction.
\end{theorem}

\begin{proof}
We verify the Markovian ODE conditions in five steps.

\textit{Step 1: exact Markov stochastic-approximation form.}
Let $\bm x_t$ be the feature in the identifiable coordinate,
$\bm\psi_t=\bm x_t-\gamma\bm x_{t+1}$, and
$X_t=F_t\rho_t$.  Use the augmented driver
$Y_{t+1}=(S_t,A_t,S_{t+1},F_t,R_{t+1})$.
For $y=(s,a,s',f,r)$ define
\begin{align}
 \bm G_c(y)&=
 \begin{bmatrix}
  X\bm x\bm\psi^\trans&\bm x\\
  X\bm\psi^\trans&1+c
 \end{bmatrix},
 &
 \bm g(y)&=Xr\begin{bmatrix}\bm x\\1\end{bmatrix},
 \label{eq:sample-G-g-new}
\end{align}
where $X=f\pi(a\mid s)/\mu(a\mid s)$.  Direct substitution into
Equation~\ref{eq:retd-update-new} gives
$\bm z_{t+1}=\bm z_t+\alpha_tH(\bm z_t,Y_{t+1})$, where
$H(\bm z,y)=\bm g(y)-\bm G_c(y)\bm z$.
No independence or square-integrability of $F_t$ has been inserted.

\textit{Step 2: invariant law and one common averaging event.}
Lemma 25 of \citet{liu2025ode}, which packages Theorems 3.2 and 3.3 of
\citet{yu2015convergence}, gives the ETD augmented chain a unique invariant
probability measure and an almost-sure law of large numbers for every
finite-dimensional function that is Lipschitz in the trace coordinate.  Our
$\lambda=0$ driver is the corresponding state--action--follow-on-trace chain.
The bounded deterministic reward in
Assumption~\ref{ass:common-theory-new} is a measurable graph of
$(S_t,A_t,S_{t+1})$, so adjoining it neither changes the invariant marginal
nor creates another invariant law.  Denote the resulting invariant law by
$\zeta$.

Every entry of $\bm G_c(y)$ and $\bm g(y)$ is affine in $f$.  Moreover,
$L_c(y)=1+\|\bm G_c(y)\|_2$ is Lipschitz in $f$, because the reverse triangle inequality bounds the
difference of two matrix norms by the norm of their difference.  Lemma 25
therefore gives finite invariant expectations and a law of large numbers for
every scalar entry of $\bm G_c$, every coordinate of $\bm g$, and $L_c$.
There are only finitely many such functions.  Intersecting their probability-
one events yields a single event $\Omega_0$ on which all these averages
converge simultaneously.  In particular, on $\Omega_0$,
\begin{equation*}
 \E_\zeta[\bm G_c(Y)]=\bm G_c,
 \qquad
 \E_\zeta[\bm g(Y)]=\bm h,
 \qquad
 \bar H(\bm z)=\bm h-\bm G_c\bm z.
\end{equation*}

\textit{Step 3: Liu Assumptions 2--4 and 6.}
The harmonic schedule verifies Liu Assumption 2: it is positive and
nonincreasing, $\sum_t\alpha_t=\infty$, $\alpha_t\to0$, and
$(\alpha_t-\alpha_{t+1})/\alpha_t=(t+B_2+1)^{-1}=O(\alpha_t)$.

For a scaling variable $r\ge1$,
\begin{equation}
 H_r(\bm z,y)=\frac{H(r\bm z,y)}r
 =-\bm G_c(y)\bm z+\frac{\bm g(y)}r.
 \label{eq:scaling-limit-new}
\end{equation}
Thus Liu Assumption 3 holds with
$H_\infty(\bm z,y)=-\bm G_c(y)\bm z$, $\kappa(r)=r^{-1}$,
$b(\bm z,y)=\bm g(y)$, and $L_b(y)=0$.
Both $H$ and $H_\infty$ are Lipschitz in $\bm z$ with envelope $L_c(y)$,
whose invariant expectation is finite by Step 2.  This verifies Liu
Assumption 4.

It remains to verify Liu Assumption 6 without taking an uncountable
intersection of null sets.  On the single event $\Omega_0$ from Step 2,
entrywise convergence gives, simultaneously for every $\bm z$,
\begin{align}
 \frac1n\sum_{k=1}^n
 \{H(\bm z,Y_k)-\bar H(\bm z)\}
 &=\frac1n\sum_{k=1}^n\{\bm g(Y_k)-\bm h\}
 \notag\\
 &\quad-
 \left\{\frac1n\sum_{k=1}^n
 [\bm G_c(Y_k)-\bm G_c]\right\}\bm z
 \longrightarrow\bm0.
 \label{eq:common-null-set-lln-new}
\end{align}
The same event covers $L_b=0$ and $L_c$.  Because
$\alpha_n=O(n^{-1})$, multiplying the centered partial sums in
Equation~\ref{eq:common-null-set-lln-new} by $\alpha_n$ makes them converge to
zero.  Liu Assumption 6 is therefore satisfied for all $\bm z$ on one
probability-one event; no square-integrable martingale-difference assumption
has been inserted.

\textit{Step 4: Liu Assumption 5 and the role of $c$.}
Stationary averaging of Equation~\ref{eq:scaling-limit-new} gives
$\bar H_r(\bm z)=-\bm G_c\bm z+\bm h/r\to-\bm G_c\bm z$
uniformly on every compact set.  The limiting mean ODE is
$\dot{\bm z}=-\bm G_c\bm z$.  Write
$\bm z=[\bm\theta^\trans,\omega]^\trans$.  Assumption~\ref{ass:c-theory-new}
gives $\Delta_c>0$, and direct completion of the square yields
\begin{align*}
 \bm z^\trans\Sym(\bm G_c)\bm z
 &=\bm\theta^\trans\bm H\bm\theta
   +2\omega\bm q^\trans\bm\theta+(1+c)\omega^2 \notag\\
 &=\bigl(\bm\theta+\bm H^{-1}\bm q\,\omega\bigr)^\trans
   \bm H
   \bigl(\bm\theta+\bm H^{-1}\bm q\,\omega\bigr)
   +\Delta_c\omega^2>0
\end{align*}
for every nonzero $\bm z$.  Hence
$\Sym(\bm G_c)\succ0$.  If $\lambda$ is an eigenvalue of $\bm G_c$ with
right eigenvector $\bm v\ne0$, then
$\Re(\lambda)=\bm v^*\Sym(\bm G_c)\bm v/(\bm v^*\bm v)>0$.
Thus $\bm G_c$ is positive stable and $-\bm G_c$ is Hurwitz.  Consequently,
the finite matrix $\bm P_c=\int_0^\infty
e^{-\bm G_c^\trans s}e^{-\bm G_cs}\,\mathrm ds\succ0$ satisfies
$\bm G_c^\trans\bm P_c+\bm P_c\bm G_c=\bm I$.
Along every nonzero solution,
\[
 \frac{\mathrm d}{\mathrm dt}(\bm z^\trans\bm P_c\bm z)
 =-\|\bm z\|_2^2
 \le-\frac{\bm z^\trans\bm P_c\bm z}
          {\lambda_{\max}(\bm P_c)}.
\]
Thus the scaling-limit origin is globally exponentially stable, which
completes Liu Assumption 5.

\textit{Step 5: boundedness and identification of the limit.}
Steps 1--4 verify Liu Assumptions 1--6 in their stated order: Step 2 gives
Assumption 1, the harmonic schedule gives Assumption 2, Step 3 gives
Assumptions 3, 4, and 6, and Step 4 gives Assumption 5.  Theorem 7 of
\citet{liu2025ode} therefore yields almost-sure boundedness.  Corollary 8 of
the same paper places the limit set inside a bounded invariant set of
$\dot{\bm z}=\bm h-\bm G_c\bm z$.  Translating by
$\bm z_c^*=\bm G_c^{-1}\bm h$ reduces this ODE to the globally exponentially
stable system in Step 4.  Its only such invariant limit set is the singleton
$\{\bm z_c^*\}$, which proves Equation~\ref{eq:diminishing-conclusion-new}.

In the original rank-deficient coordinate, use the matrix $\bm B$ in
Equation~\ref{eq:identifiable-basis-new}.  Every parameter update lies in
$\operatorname{col}(\bm B)=\operatorname{row}(\bm\Phi)$, so the orthogonal
parameter-null component is conserved; the scalar auxiliary coordinate is
unchanged by this reduction.  Multiplication by $\bm\Phi$ then gives the stated
prediction-space limit.
\end{proof}

The shrinking stepsize is essential to this theorem.  Over any fixed late
window of length $K$,
$0\le1-\alpha_{t+j}/\alpha_t=j/(t+B_2+j)\le K/(t+B_2)$ for $0\le j<K$,
but local near-constancy does not turn the limit $\alpha_t\to0$ into a theorem
for a constant positive $\alpha$.

\subsection{Constant-Stepsize Stochastic Stability}

For a nonzero constant stepsize, the driver must control an unbounded matrix
coefficient over arbitrarily long time intervals.  We state those additional
conditions explicitly.  They concern the Markov trace process and matrix
growth, not the sign of the RETD random product.

\begin{assumption}[Super-Lyapunov driver condition]
\label{ass:driver-new}
The augmented driver $Y_t$ is irreducible and aperiodic.  There exist
$b_0,B_0>0$, $\delta\in(1/2,1]$, a measurable
$V:\mathsf Y\to[e,\infty)$, $W=\log V$, and a small set $C_0$ such that
\begin{equation*}
 PV(y)\le e^{-b_0W(y)^\delta}V(y)\bm1_{C_0^c}(y)
          +B_0\bm1_{C_0}(y),
\end{equation*}
and every finite sublevel set of $W$ is small.
\end{assumption}

\begin{assumption}[Growth of the RETD sample matrix]
\label{ass:growth-new}
For some $C_{G,0}<\infty$, $\epsilon\in(0,1)$, and
$\beta_G<\min\{2\delta-1,\delta/(1+\epsilon)\}$, every entry satisfies
\begin{align}
 |[\bm G_0(y)]_{ij}|&\le C_{G,0}W(y)^{\beta_G},\notag\\
 |[\bm G_c(y)]_{ij}|&\le C_G(c)W(y)^{\beta_G},
 \qquad C_G(c)=C_{G,0}+c.
 \label{eq:growth-Gc-new}
\end{align}
\end{assumption}

For the affine conclusion define $\bm z_c^*=\bm G_c^{-1}\bm h$ and
$\bm\xi_t^R=\bm g(Y_t)-\bm G_c(Y_t)\bm z_c^*$.

\begin{assumption}[Affine residual moment]
\label{ass:residual-new}
For the required $p\ge1$,
$\E_\zeta[V(Y_0)^{1/2}\|\bm\xi_0^R\|_2^p]<\infty$, where $\zeta$ is the
invariant law of the RETD driver.  We avoid the symbol
$\pi$ here because it already denotes the target policy.
\end{assumption}

\subsection{Checkability and Coverage of the Constant-Step Assumptions}

Assumptions~\ref{ass:driver-new}--\ref{ass:growth-new} are not automatic in a
finite MDP.  The augmented state contains the unbounded follow-on trace, so the
driver is a general-state-space Markov chain even when $\mathcal S$ and
$\mathcal A$ are finite.  A conservative directly checkable route is the
uniformly contractive condition $\gamma\rho(s,a)\le r<1$ for every covered
$(s,a)$, combined with a geometric drift/minorization certificate for the resulting
bounded augmented driver.  Then
$F_t\le\max\{F_0,(1-r)^{-1}\}$, the sample matrices are bounded, and the
matrix-growth part becomes immediate.  This sufficient condition does not
cover the trace-amplifying runs used by the two counterexamples.

The heavy-tailed two-state and Baird instances are analyzed outside this
small-stepsize route because explicit super-Lyapunov constants
$b_0,B_0,\delta,V$ are unavailable for their augmented drivers.  Their product
signs instead follow from environment-specific regenerative or projective
certificates.  By contrast, a bounded-trace finite MDP satisfying the
uniformly contractive condition above is covered once its standard
drift/minorization constants are supplied.  This separation prevents the
general theorem from being claimed for the two stress tests without the
required driver certificate.

The external tool is Theorem 1 of \citet{durmus2021stability}.  Under their
UE1 driver condition, A1 matrix-growth condition, and the Hurwitz condition
for the negative stationary mean matrix, that theorem gives, for each
$p\ge1$, a constructive positive sufficient threshold
$\alpha_{\infty,p}$ such that
\begin{equation*}
 \left(\E_y\left\|
 \prod_{k=m+1}^{n}(\bm I-\alpha_k\mathcal A(Y_k))
 \right\|_2^p\right)^{1/p}
 \le C_p\exp\left[-\frac{a}{4}
       \sum_{k=m+1}^n\alpha_k\right]V(y)^{1/(2p)}.
\end{equation*}
Here $\bm Q\succ0$ solves
$\overline{\mathcal A}^\trans\bm Q+
 \bm Q\overline{\mathcal A}=\bm I$ and
$a=(2\|\bm Q\|_2)^{-1}$.  The proof below identifies the RETD driver, matrix
family, stationary mean, Lyapunov matrix, and growth constant required by the
quoted theorem.  In particular, the substitution is
\begingroup
\centering
\footnotesize
\setlength{\tabcolsep}{4pt}
\renewcommand{\arraystretch}{.95}
\begin{tabular}{ll}
\toprule
Durmus et al. object & RETD object \\
\midrule
$Z_k,P,\pi$ & $Y_k,P_Y,\zeta$ \\
$\overline{\mathcal A}(Z_k)$ & $\bm G_c(Y_k)$ \\
$\mathcal A=\E_\pi[\overline{\mathcal A}(Z_0)]$ &
  $\bm G_c=\E_\zeta[\bm G_c(Y_0)]$ \\
$V,W,\delta$ and UE1 constants & those in Assumption~\ref{ass:driver-new} \\
$C_A,\beta$ in A1 & $C_G(c),\beta_G$ in Assumption~\ref{ass:growth-new} \\
$\bm Q,\kappa_Q,a$ & $\bm Q_c,\kappa_{Q_c},a_c$ \\
\bottomrule
\end{tabular}
\par
\endgroup
The block length $h$ and the threshold $\alpha_{\infty,p}$ are then exactly
those in Equations (89) and (90) of \citet{durmus2021stability} after this
substitution.  This table is also the definition used below; no additional
undefined threshold function is introduced.

\begin{theorem}[Constant-stepsize RETD]
\label{thm:retd-fixed-new}
Let Assumptions~\ref{ass:common-theory-new}--\ref{ass:growth-new} hold.
Fix $p\ge1$, and let $\bm Q_c\succ0$ be the unique solution of
\begin{equation}
 \bm G_c^\trans\bm Q_c+\bm Q_c\bm G_c=\bm I,
 \qquad
 \kappa_{Q_c}=\frac{\lambda_{\max}(\bm Q_c)}
                    {\lambda_{\min}(\bm Q_c)},
 \qquad a_c=(2\|\bm Q_c\|_2)^{-1}.
 \label{eq:Qc-new}
\end{equation}
Let $\alpha_{D,p}^R(c)$ be the positive number
$\alpha_{\infty,p}$ in Equation (90) of \citet{durmus2021stability}, with
the complete substitution displayed above and with its block length $h$
defined by their Equation (89):
\begin{equation*}
 \alpha_{D,p}^R(c):=
 \alpha_{\infty,p}\big|_{
 Z=Y,\,\overline{\mathcal A}=\bm G_c(\cdot),\,
 \mathcal A=\bm G_c,\,\bm Q=\bm Q_c,\,C_A=C_G(c)}.
\end{equation*}
The number is computable after the UE1 drift and minorization constants are
supplied, is sufficient rather than sharp, and exists by the quoted theorem.
Define
\begin{equation}
 \overline\alpha_p^R(c)=
 \min\left\{
  \alpha_{D,p}^R(c),
  \frac{2\underline m_c}{\|\bm G_c\|_2^2},
  \frac{2}{1+c}
 \right\}.
 \label{eq:alpha-bar-new}
\end{equation}
The first term is the external random-product threshold.  The second and third
terms are additional transparent RETD restrictions: they make the
deterministic mean step contractive and give the isolated auxiliary retention
$|1-\alpha(1+c)|<1$.  They are not assumptions of
Theorem 1 of \citet{durmus2021stability}.
If $0<\alpha<\overline\alpha_p^R(c)$, put
\begin{equation*}
 \bm A_t(c)=\bm I-\alpha\bm G_c(Y_t),
 \qquad
 \bm\Phi_R(t,s;c)=\bm A_t(c)\cdots\bm A_s(c),\quad s\le t,
 \qquad
 \bm\Phi_R(t,t+1;c)=\bm I.
\end{equation*}
Then, with time increasing from right to left,
\begin{equation}
 \left(\E_y\|\bm\Phi_R(n,1;c)\|_2^p\right)^{1/p}
 \le C_{c,p}e^{-(a_c/4)\alpha n}V(y)^{1/(2p)}.
 \label{eq:retd-moment-contraction-new2}
\end{equation}
Consequently,
\begin{equation}
 \lambda_R(\alpha,c)
 :=\limsup_{n\to\infty}\frac1n\log\|\bm\Phi_R(n,1;c)\|_2
 \le-\frac{a_c\alpha}{4}
 =-\frac{\alpha}{8\|\bm Q_c\|_2}<0
 \quad\text{a.s.}
 \label{eq:retd-top-negative-new2}
\end{equation}
If Assumption~\ref{ass:residual-new} also holds on a two-sided stationary
extension of the driver, the centered affine recursion has the unique finite-
$p$th-moment stationary causal solution
\begin{equation}
 \bm e_t^\circ=
 \alpha\sum_{j=0}^{\infty}
 \bm\Phi_R(t,t-j+1;c)\bm\xi_{t-j}^R,
 \label{eq:stationary-series-new2}
\end{equation}
and every two solutions driven by the same samples forget their initial
conditions exponentially.  If $\bm\xi_t^R\equiv\bm0$ sample by sample, the
iterate itself converges exponentially to $\bm z_c^*$ almost surely.
\end{theorem}

\begin{proof}
\textit{Step 1: the regularizer changes each sampled product.}
From Equation~\ref{eq:sample-G-g-new},
\begin{equation}
 \bm G_c(y)=\bm G_0(y)+c\bm e_\omega\bm e_\omega^\trans.
 \label{eq:sample-c-rank-one-new}
\end{equation}
Thus the product in the theorem depends on $c$.  On a
zero-importance-ratio sample, its auxiliary diagonal entry is
$1-\alpha(1+c)$.  The final term in
Equation~\ref{eq:alpha-bar-new} ensures that this scalar mode has magnitude
less than one.

\textit{Step 2: $c$ creates a verified mean margin.}
Assumption~\ref{ass:c-theory-new} makes $\Delta_c>0$.  The congruence in
Equation~\ref{eq:retd-schur-congruence-new} therefore proves
$\Sym(\bm G_c)\succ0$ and hence that $-\bm G_c$ is Hurwitz.  Furthermore,
Equation~\ref{eq:mc-new} gives
$\bm y^\trans\Sym(\bm G_c)\bm y\ge
\underline m_c\|\bm y\|_2^2$ for every $\bm y$.
Expanding a deterministic mean step yields
\begin{align*}
 \|(\bm I-\alpha\bm G_c)\bm y\|_2^2
 &=\|\bm y\|_2^2
   -2\alpha\bm y^\trans\Sym(\bm G_c)\bm y
   +\alpha^2\|\bm G_c\bm y\|_2^2\\
 &\le\{1-2\alpha\underline m_c+
       \alpha^2\|\bm G_c\|_2^2\}\|\bm y\|_2^2.
\end{align*}
The second term in Equation~\ref{eq:alpha-bar-new} makes this multiplier
strictly less than one.  The same Hurwitz conclusion verifies condition A2 of
the quoted random-product theorem, without assuming contraction of the sampled
product.

\textit{Step 3: $c$ enters the rate and the stochastic threshold.}
Since $-\bm G_c$ is Hurwitz,
$\bm Q_c=\int_0^\infty e^{-\bm G_c^\trans s}e^{-\bm G_cs}\,\mathrm ds$ is the
unique positive-definite solution of
Equation~\ref{eq:Qc-new}.  Hence $c$ changes $\bm Q_c$ and the decay constant
$a_c$.  Independently, Equation~\ref{eq:growth-Gc-new} shows that $c$ changes
the one-step growth constant $C_G(c)$.  Assumption~\ref{ass:driver-new}
verifies UE1, Assumption~\ref{ass:growth-new} verifies A1, and Step 2 verifies
A2 for the actual RETD matrix family.  The external theorem therefore applies
at every constant stepsize below $\alpha_{D,p}^R(c)$ and gives
Equation~\ref{eq:retd-moment-contraction-new2}.

\textit{Step 4: moment contraction gives a negative top exponent.}
Fix a rational $\varepsilon\in(0,a_c\alpha/4)$.  Markov's inequality and
Equation~\ref{eq:retd-moment-contraction-new2} give
\[
 \Prob_y\!\left(
  \|\bm\Phi_R(n,1;c)\|_2>
  e^{-(a_c\alpha/4-\varepsilon)n}\right)
 \le C_{c,p}^pV(y)^{1/2}e^{-p\varepsilon n}.
\]
The right-hand side is summable.  Borel--Cantelli implies that the displayed
event occurs only finitely often.  Taking the intersection over the countable
set of positive rational $\varepsilon<a_c\alpha/4$ and then sending
$\varepsilon\downarrow0$ proves
Equation~\ref{eq:retd-top-negative-new2} on one probability-one event.

\textit{Step 5: affine stability and uniqueness.}
Let $\chi=a_c\alpha/4$.  For $j\ge1$, condition first on $Y_{t-j}$.
The Markov property and Equation~\ref{eq:retd-moment-contraction-new2} give
\begin{align*}
 &\E_\zeta\!\left[
  \|\bm\Phi_R(t,t-j+1;c)\bm\xi_{t-j}^R\|_2^p
  \right] \notag\\
 &\quad\le
 C_{c,p}^p e^{-p\chi j}
 \E_\zeta\!\left[
 V(Y_{t-j})^{1/2}\|\bm\xi_{t-j}^R\|_2^p
 \right].
\end{align*}
The expectation on the right is finite and independent of $t,j$ by
Assumption~\ref{ass:residual-new} and stationarity.  The $j=0$ term is
$\bm\xi_t^R$ because
$\bm\Phi_R(t,t+1;c)=\bm I$, and it is also in $L^p$.  Hence
$\sum_{j=0}^{\infty}
\|\bm\Phi_R(t,t-j+1;c)\bm\xi_{t-j}^R\|_{L^p}<\infty$.
Minkowski's inequality proves that
Equation~\ref{eq:stationary-series-new2} converges in $L^p$.  Since $p\ge1$,
the same summability also gives
$\sum_j\E\|\bm\Phi_R\bm\xi^R\|_2<\infty$; Tonelli's theorem therefore gives
absolute convergence almost surely.  The limit is causal and stationary.
Separating
the $j=0$ term and using
$\bm\Phi_R(t,t-j+1;c)=
 \bm A_t(c)\bm\Phi_R(t-1,t-j+1;c)$ for $j\ge1$ shows directly that
$\bm e_t^\circ=\bm A_t(c)\bm e_{t-1}^\circ+\alpha\bm\xi_t^R$,
so it is a stationary solution of the centered affine recursion.

For any two solutions driven by the same two-sided sample path,
$\bm e_t-\widetilde{\bm e}_t=\bm\Phi_R(t,s+1;c)
(\bm e_s-\widetilde{\bm e}_s)$ for $t>s$.
Step 4 makes the norm of this difference decay with exponent at most
$-\chi$ for every finite initial difference.  If the difference is stationary,
its distribution is independent of $t$, while the displayed product makes it
converge to zero in probability; its stationary distribution must therefore
be the point mass at zero.  This proves uniqueness, while arbitrary finite
initial differences are forgotten exponentially.  If
$\bm\xi_t^R\equiv\bm0$ sample by sample, the
stationary series is identically zero and the same homogeneous contraction
gives $\bm z_t\to\bm z_c^*$ exponentially.

Finally, increasing $c$ enlarges $\Delta_c$ but also changes
$\|\bm G_c\|_2$, $\bm Q_c$, $C_G(c)$, and the scalar factor
$1-\alpha(1+c)$.  The theorem therefore yields the feasible set
$\mathcal C_{\alpha,p}^R=
\{c>c_{\min}:\alpha<\overline\alpha_p^R(c)\}
\subset(c_{\min},2/\alpha-1)$, which need not expand monotonically with $c$.
\end{proof}

\subsection{Comparison with ETD under the Constant-Stepsize Theorem}
\label{sec:etd-after-theorem}

The external theorem of \citet{durmus2021stability} is algorithm agnostic;
Theorem~\ref{thm:retd-fixed-new} is not.  ETD and RETD supply different
matrix families:
\begin{equation}
 \mathcal A_E(Y_t)=X_t\bm x_t\bm\psi_t^\trans,
 \qquad
 \mathcal A_R(Y_t;c)=
 \begin{bmatrix}
 X_t\bm x_t\bm\psi_t^\trans&\bm x_t\\
 X_t\bm\psi_t^\trans&1+c
 \end{bmatrix}.
 \label{eq:etd-retd-matrices-new}
\end{equation}
Consequently they have different dimensions, sample products, stationary mean
matrices, Lyapunov equations, growth constants, and admissible stepsize
thresholds.  ETD must therefore be analyzed with its own matrix family; it is
not a principal submatrix of the RETD recursion.

ETD satisfies the same \emph{external analytical tool} on a particular task
when its augmented trace driver verifies UE1, its entries
satisfy the corresponding growth bound, $-\bm A$ is Hurwitz, and
$\alpha$ is below the threshold obtained from
$\mathcal A_E$.  Under these conditions ETD has moment contraction and a
negative top exponent, which covers sufficiently small stepsizes in suitable
environments.

There are three distinct reasons an ETD certificate can fail.  First,
$-\bm A$ may not be Hurwitz.  Second, the unbounded
trace driver may fail the required drift or growth condition.  Third, even
when the first two requirements do hold, the specified stepsize may still lie
outside ETD's own product-stability region.  The two-state counterexample
proves directly that the third requirement fails at $\alpha=0.1$: ETD's exact
scalar exponent is positive, so no valid negative-exponent theorem can cover
that matrix product at that stepsize.  The example also has
$\E F_t^2=\infty$, and we have not supplied UE1 or A1 constants for its
augmented driver.  It would therefore be incorrect to claim that the first two
external conditions have been verified there.  The Baird stress test has the
same coverage boundary: its signs are established by the separate analysis in
Section~\ref{sec:baird}, not by
Theorem~\ref{thm:retd-fixed-new}.  RETD does not use the ETD matrices; it uses
the block matrices in
Equation~\ref{eq:etd-retd-matrices-new}, and at $c=4$ their complete cycle
product has a negative exponent.

When $\rho_t=0$, ETD supplies $\bm I$, while RETD supplies
Equation~\ref{eq:retd-zero-ratio-new}.  The
parameter $c$ therefore changes an actual post-shock response, then propagates
through $\Delta_c$, $\bm Q_c$, $C_G(c)$, and
$\overline\alpha_p^R(c)$.  These quantities determine separate stability
regions for ETD and RETD.

\section{Numerical Study of ETD and RETD on Baird's Counterexample}
\label{sec:baird}

We analyze Baird's seven-state geometry using the same quantities as in the
two-state counterexample: regenerative trace cycles, deterministic mean maps,
sampled cycle matrices, and prediction-relevant top Lyapunov exponents.

\subsection{The Counterexample and the Identifiable Prediction Space}

There are six upper states $u_1,\ldots,u_6$ and one lower state $\ell$.
From every state, the solid action moves to $\ell$, whereas the dashed action
moves uniformly to an upper state.  The target policy always selects solid.
The behavior policy selects solid with probability $1/7$ and dashed with
probability $6/7$.  Hence $\rho_t=7$ on solid, $\rho_t=0$ on dashed, and
$d_\mu(s)=1/7$.

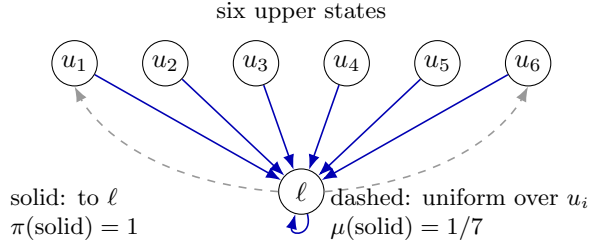
\begin{figure}[t]
\centering
\begin{tikzpicture}[>=Latex,
  state/.style={circle,draw,minimum size=6mm,inner sep=0pt,font=\small},
  solid/.style={->,blue!70!black,semithick},
  dasharrow/.style={->,gray!80,dashed,semithick}]
  \node[state] (u1) at (-3.00,0.65) {$u_1$};
  \node[state] (u2) at (-1.80,0.65) {$u_2$};
  \node[state] (u3) at (-0.60,0.65) {$u_3$};
  \node[state] (u4) at (0.60,0.65) {$u_4$};
  \node[state] (u5) at (1.80,0.65) {$u_5$};
  \node[state] (u6) at (3.00,0.65) {$u_6$};
  \node[state] (ell) at (0,-1.05) {$\ell$};
  \foreach \upper in {u1,u2,u3,u4,u5,u6}
    \draw[solid] (\upper) -- (ell);
  \draw[solid] (ell) edge[loop below,looseness=5] (ell);
  \draw[dasharrow] (ell.west) .. controls (-2.25,-0.90) and (-2.85,0.05) .. (u1.south);
  \draw[dasharrow] (ell.east) .. controls (2.25,-0.90) and (2.85,0.05) .. (u6.south);
  \node[font=\footnotesize,align=center] at (0,1.33) {six upper states};
  \node[font=\footnotesize,align=left] at (-3.00,-1.35)
    {solid: to $\ell$\\$\pi(\mathrm{solid})=1$};
  \node[font=\footnotesize,align=left] at (2.10,-1.35)
    {dashed: uniform over $u_i$\\$\mu(\mathrm{solid})=1/7$};
\end{tikzpicture}
\caption{Baird's seven-state counterexample.  The behavior policy uses the
blue solid transition with probability $1/7$ and otherwise uses the dashed
uniform-to-upper transition.  The dashed arcs denote the common kernel from
every state.}
\label{fig:baird-mdp}
\end{figure}

We use $\gamma=0.99$, zero rewards, and the standard features
$\bm\phi(u_i)=2\bm e_i+\bm e_8$ and
$\bm\phi(\ell)=\bm e_7+2\bm e_8$.  The $7\times8$ feature matrix has rank
seven, with
\[
 \ker\bm\Phi=\operatorname{span}
 \{(-1,-1,-1,-1,-1,-1,-4,2)^\trans\}.
\]
Thus an eight-dimensional parameter product contains one prediction-null
direction.  We remove it and work in the seven-dimensional prediction
coordinate $\bm y=\bm\Phi\bm\theta$.  Because rewards are zero, the target
value is exactly $\bm0$ and both error recursions are homogeneous.

\subsection{The Common Regenerative Cycle}

A dashed action has zero importance ratio and resets the next follow-on trace
to one.  Let $L$ be the number of consecutive solid actions before the next
dashed action.  Successive cycles are independent, with
$\Pr(L=\ell)=(6/7)(1/7)^\ell$ and $\E[L+1]=7/6$.
Within a solid run,
\begin{equation}
 F_0=1,\qquad F_{k+1}=1+6.93F_k,\qquad
 F_k=\frac{6.93^{k+1}-1}{5.93}.
 \label{eq:baird-exact-trace-recurrence}
\end{equation}
Consequently, $\E[F_t^r]<\infty$ if and only if $6.93^r/7<1$.
The threshold $\log7/\log6.93\approx1.00519$ lies between one and two:
the trace has a finite mean and infinite variance.  This reproduces the shock
source in the two-state counterexample, but it still does not establish the
sign of either complete matrix product.

The effect of a long run is explicit.  On a lower-to-lower solid transition,
the scalar multiplier in the lower-prediction direction is
$d_k=1-\alpha\rho F_k(1-\gamma)\|\bm\phi(\ell)\|^2
=1-0.35\alpha F_k$.
Since $F_k\asymp6.93^k$, a length-$L$ run eventually accumulates
$\sum_{k<L}\log|d_k|=\tfrac12\log(6.93)L^2+O(L)$.  Whether the algorithm
contracts or expands depends on how the final dashed sample completes this
cycle.

\subsection{ETD: Stable Mean Dynamics, Expanding Sampled Dynamics}

The emphatic weights solve
$\bm f^\trans(\bm I-\gamma\bm P_\pi)=\bm d_\mu^\trans$ and are
$f(u_i)=1/7$ and $f(\ell)=694/7$.
In the identifiable coordinate, the ETD mean matrix has five eigenvalues
$4/7$ and two eigenvalues
$(1641\pm37\sqrt{649})/700=0.99772471,3.69084672$.  Therefore, at
$\alpha=0.01$, $\rho(\bm I-\alpha\bm A_E)=0.99428571<1$.
The deterministic ETD mean map is strictly contractive.

To analyze the samples, define $\bm K=\bm\Phi\bm\Phi^\trans$,
$\bm k_s=\bm K\bm e_s$, $\bm v_{s,s'}=\bm e_s-\gamma\bm e_{s'}$, and
$X_t=F_t\rho_t$.
The prediction error obeys
\begin{equation*}
 \bm y_{t+1}^E=\bm M_t^E\bm y_t^E,\qquad
 \bm M_t^E
 =\bm I-\alpha X_t\bm k_{S_t}
                  \bm v_{S_t,S_{t+1}}^\trans .
\end{equation*}
Let $\bm M_u^E(i,F_0)$ denote the first upper-to-lower solid factor and
$\bm M_\ell^E(F_k)$ a later lower-to-lower solid factor.  The final dashed
ETD factor is the identity because $X_t=0$.  Hence
\begin{equation}
 \bm C_E(i,0)=\bm I,\qquad
 \bm C_E(i,L)=\bm M_\ell^E(F_{L-1})\cdots
 \bm M_\ell^E(F_1)\bm M_u^E(i,F_0),\quad L\ge1 .
 \label{eq:baird-etd-cycle}
\end{equation}
The prediction-relevant exponent is
\begin{equation*}
 \lambda_E^{\mathrm{pred}}
 =\lim_{N\to\infty}
 \frac{\log\|\bm C_{E,N}\cdots\bm C_{E,1}\|}
      {\sum_{n=1}^N(L_n+1)} .
\end{equation*}
The geometric cycle law above and the quadratic
log-growth bound above give the logarithmic moments required for this limit.

At $\alpha=0.01$, a $10{,}000$-transition burn-in followed by $200{,}000$
retained transitions in each of ten independent runs gives
$\widehat\lambda_E^{\mathrm{pred}}=0.00138528\pm0.00008328$,
where the uncertainty is the sample standard deviation across runs; all ten
estimates are positive.  This is an empirical estimate and provides no
analytic lower bound.  The mechanism is the
same as in Section~\ref{sec:two-state-etd-product}:
long solid runs create large factors, and the subsequent zero-ratio ETD
factor is exactly $\bm I$, so it cannot release a post-shock correction.

\subsection{RETD: The Same Cycle Includes Recovery}

For RETD the joint prediction error is
$[(\bm y_t^R)^\trans,\omega_t]^\trans$, with sample matrix
\begin{equation*}
 \bm M_t^R(c)=\bm I-\alpha
 \begin{bmatrix}
  X_t\bm k_{S_t}\bm v_{S_t,S_{t+1}}^\trans & \bm k_{S_t}\\
  X_t\bm v_{S_t,S_{t+1}}^\trans & 1+c
 \end{bmatrix}.
\end{equation*}
Its Baird mean matrix satisfies $\Sym(\bm G_c)\succ0$ if and only if
$c>29403/109397$, so $c=0.5$ is inside the certified region.  At
$(\alpha,c)=(0.01,0.5)$,
\[
 \min\operatorname{Re}\spec(\bm G_c)=0.24794122,
 \qquad \rho(\bm I-\alpha\bm G_c)=0.99752059<1.
\]
The RETD mean map is stable, but it is slower than the ETD mean map reported
above.

On the final dashed sample, RETD applies
\begin{equation*}
 \bm D_s(c)=\bm I-\alpha
 \begin{bmatrix}
  \bm0&\bm k_s\\
  \bm0^\trans&1+c
 \end{bmatrix},
\end{equation*}
which is not the identity.  Its complete cycles are
\begin{align*}
 \bm C_R(i,0;c)&=\bm D_{u_i}(c),\notag\\
 \bm C_R(i,L;c)&=\bm D_\ell(c)
 \bm M_\ell^R(F_{L-1};c)\cdots
 \bm M_\ell^R(F_1;c)\bm M_u^R(i,F_0;c),\quad L\ge1 .
\end{align*}
The only structural change relative to Equation~\ref{eq:baird-etd-cycle} is
precisely the change identified in the two-state counterexample: RETD uses the
zero-ratio sample to release the stored correction.

The same numerical setting gives
$\widehat\lambda_R^{\mathrm{pred}}=-0.00143906\pm0.00006106$,
and all ten estimates are negative.  Here the negative sign also admits a
rigorous cycle certificate.

\begin{proposition}[Certified Baird RETD contraction]
\label{cor:baird-retd-certified-contraction}
For Baird RETD in the identifiable joint prediction coordinate at
$\alpha=0.01$ and $c=0.5$,
\begin{equation*}
 \lambda_R^{\mathrm{pred}}\le-10^{-4}<0 .
\end{equation*}
\end{proposition}

\begin{proof}
The dashed times generate iid cycles with probabilities
$p_L=(6/7)(1/7)^L$ and mean length $\E[L+1]=7/6$.  We first verify that no
logarithm below is applied to zero.  The matrix determinant lemma gives
\begin{equation*}
 \det\bm M_t^R(c)
 =(1-\alpha c)
   \left(1-\alpha X_t\bm v_{S_t,S_{t+1}}^\trans\bm k_{S_t}\right)-\alpha .
\end{equation*}
At $(\alpha,c,\gamma)=(1/100,1/2,99/100)$, a dashed factor has determinant
$197/200$.  An upper-to-lower solid factor has determinant
$774657/1000000$.  For a lower-to-lower solid factor,
$\bm v_{\ell,\ell}^\trans\bm k_\ell=1/20$ and $X=7F_k$, so
\begin{equation*}
 \det\bm M_\ell^R(F_k;1/2)
 =\frac{394000-1393F_k}{400000}.
\end{equation*}
It could vanish only if $F_k=394000/1393$.  The fraction is reduced and
$1393=7\cdot199$, whereas the recurrence in
Equation~\ref{eq:baird-exact-trace-recurrence} gives
$F_k\in\mathbb Z[1/10]$, whose reduced denominator contains only the primes
$2$ and $5$.  Hence every sample and cycle matrix is nonsingular.

We next control every projective direction, rather than a finite set of
sampled initial vectors.  Use the following positive-definite matrix $\bm P$,
invariant under permutations of the six upper states, whose independent entry
classes are
\[
 \begin{aligned}
  P_{ii}&=0.7473190564759538\quad(i\le6),
  &P_{ij}&=0.13894231222775422\quad(i\ne j\le6),\\[-1mm]
  P_{i\ell}&=-1.4462721104545542,
  &P_{i\omega}&=1.3688386146118432,\\[-1mm]
  P_{\ell\ell}&=9.245683770562014,
  &P_{\ell\omega}&=-9.008824250903285,\\[-1mm]
  P_{\omega\omega}&=11.073004906486352.
 \end{aligned}
\]
Every displayed decimal is interpreted as a rational number.  Exact
symmetry-reduced Sylvester tests give
\begin{equation}
 \frac3{50}\bm I\prec\bm P\prec21\bm I .
 \label{eq:baird-P-bounds}
\end{equation}
For each run length define
\begin{align*}
 \overline{\bm Q}_L
 &=\frac16\sum_{i=1}^6
   \bm C_R(i,L;1/2)^\trans\bm P\bm C_R(i,L;1/2),
 \\
 \bm v_*
 &=\bigl(0.3772374\,\bm1_6^\trans,
          0.38224657,0.00626418\bigr)^\trans,
 \qquad
 \beta_L=\frac{\bm v_*^\trans\overline{\bm Q}_L\bm v_*}
                  {\bm v_*^\trans\bm P\bm v_*}.
\end{align*}
All displayed decimals are again interpreted as exact rationals.  Put
\begin{equation*}
 m_{11}=\sum_{L=0}^{11}p_L
 =\frac{13841287200}{13841287201},
 \qquad
 \bm S_{11}=\sum_{L=0}^{11}p_L
                 \frac{\overline{\bm Q}_L}{\beta_L}.
\end{equation*}
The read-only verifier performs symmetry-reduced exact Sylvester tests and
certifies
\begin{equation}
 \left(m_{11}+10^{-10}\right)\bm P-\bm S_{11}\succ0.
 \label{eq:baird-generalized-eigenvalue-new}
\end{equation}

We spell out how this matrix inequality becomes a direction-uniform log
drift.  For a nonzero current direction $\bm u$, first average over the next
upper state $i$, which is uniform and independent of the past.  Jensen's
inequality gives
\begin{align}
 &\frac16\sum_{i=1}^6
 \log\frac{\|\bm C_R(i,L;1/2)\bm u\|_{\bm P}}
               {\|\bm u\|_{\bm P}}
 \notag\\
 &\qquad\le\frac12\log
 \frac{\bm u^\trans\overline{\bm Q}_L\bm u}
      {\bm u^\trans\bm P\bm u},
 \qquad \|\bm u\|_{\bm P}=(\bm u^\trans\bm P\bm u)^{1/2}.
 \label{eq:baird-jensen-direction-new}
\end{align}
Write the ratio on the right as $\beta_L r_L(\bm u)$ and use
$\log r\le r-1$.  Summing over $0\le L\le11$ and applying
Equation~\ref{eq:baird-generalized-eigenvalue-new} yields, uniformly in
$\bm u\ne0$, the following bound, where $g_{0:11}(\bm u)$ denotes the
$p_L$-weighted sum of the left-hand side of
Equation~\ref{eq:baird-jensen-direction-new}:
\begin{equation*}
 g_{0:11}(\bm u)
 \le\frac12\left\{
 \sum_{L=0}^{11}p_L(\overline{\log\beta_L}-1)
 +m_{11}+10^{-10}\right\},
\end{equation*}
where $\overline{\log\beta_L}$ is the outward rational upper bound obtained
from the 48-term logarithm enclosure used by the verifier.  Exact rational
evaluation gives
\begin{equation}
 g_{0:11}:=\sup_{\bm u\ne0}g_{0:11}(\bm u)
 \le-0.000211197432303078308 .
 \label{eq:baird-finite-drift}
\end{equation}
It remains to bound the unenumerated event explicitly.  Put $N=12$.  The
geometric law gives
\begin{align}
 \Pr(L\ge N)&=7^{-N},\notag\\
 \sum_{L=N}^{\infty}p_LL
 &=7^{-N}\left(N+\frac16\right),\notag\\
 \sum_{L=N}^{\infty}p_LL(L-1)
 &=7^{-N}\left\{N(N-1)+\frac N3+\frac1{18}\right\}.
 \label{eq:baird-tail-moments-new}
\end{align}
Using $F_k\le(693/593)(693/100)^k$, the factor norm bound
$\|\bm M_k^R\|_2\le(357/200)F_k$, and
Equation~\ref{eq:baird-P-bounds}, every direction has the tail majorant
\begin{align*}
 g_{12:\infty}(\bm u)
 &\le\sum_{L=N}^{\infty}p_L\left\{
 (L+1)\log\frac{357}{200}
 +L\log\frac{693}{593}\right.\notag\\
 &\hspace{34mm}\left.
 +\frac{L(L-1)}2\log\frac{693}{100}
 +\frac12\log 350\right\}.
\end{align*}
Substitution of the three exact tail sums in
Equation~\ref{eq:baird-tail-moments-new}, followed by the same outward
48-term logarithm enclosures, gives
\begin{equation*}
 g_{12:\infty}:=\sup_{\bm u\ne0}g_{12:\infty}(\bm u)
 \le1.04142204840942604\times10^{-8}.
\end{equation*}
With $\varepsilon=10^{-4}$ and $\E[L+1]=7/6$,
\begin{equation}
 g_{0:11}+g_{12:\infty}+\varepsilon\frac76
 \le-9.4520351415927536\times10^{-5}<0 .
 \label{eq:baird-total-drift}
\end{equation}
It remains to pass from a one-cycle conditional drift to the top physical-time
exponent.  Let $\Delta_n$ be the log $\bm P$-norm gain of cycle $n$ along the
direction produced by the preceding cycles, and let $\tau_n=L_n+1$.
Equations~\ref{eq:baird-finite-drift}--\ref{eq:baird-total-drift} imply
$\E[\Delta_n+\varepsilon\tau_n\mid\mathcal F_{n-1}]\le-\eta$
for a constant $\eta>0$, uniformly over the incoming nonzero direction.  The
determinant calculation above and the factor bounds give
$|\Delta_n|\le a+bL_n^2$ for finite constants $a,b$.  Since $L_n$ is
geometric, the associated martingale differences are square integrable, and
the martingale strong law gives
\[
 \limsup_{N\to\infty}
 \frac{\sum_{n=1}^N\Delta_n}{\sum_{n=1}^N\tau_n}
 \le-\varepsilon\qquad\text{a.s.}
\]
for every fixed initial direction.

Choose a fixed $1/2$-net of the $\bm P$-unit sphere.  It is finite, so the
preceding probability-one events may be intersected, and the standard net
bound makes the operator $\bm P$-norm at most twice the largest norm on the
net.  Norm equivalence from Equation~\ref{eq:baird-P-bounds} leaves the
exponent unchanged.  Finally, a partial cycle has positive log norm at most
$a'+b'L_n^2$; the geometric tail and Borel--Cantelli give
$L_n^2/n\to0$, so the unfinished cycle contributes $o(t)$.  The endpoint
bound therefore holds at every transition and proves
$\lambda_R^{\mathrm{pred}}\le-10^{-4}$.
\end{proof}

\subsection{The Sign Comparison}

\begin{table}[t]
\centering
\footnotesize
\setlength{\tabcolsep}{4pt}
\begin{tabular}{lccc}
\toprule
Method & Mean-map radius & Top-exponent estimate & Rigorous sign \\
\midrule
ETD & $0.99428571$ & $0.00138528\pm0.00008328$ & numerical only \\
RETD ($c=0.5$) & $0.99752059$ & $-0.00143906\pm0.00006106$ & $\lambda_R\le-10^{-4}$ \\
\bottomrule
\end{tabular}
\caption{Baird comparison at $\alpha=0.01$.  Mean radii are deterministic.
Exponent estimates are means plus/minus sample standard deviations over ten
independent runs.  The RETD negative sign has the uniform cycle certificate
in Proposition~\ref{cor:baird-retd-certified-contraction}; the ETD positive
sign remains numerical.}
\label{tab:baird-sign-comparison}
\end{table}

Both mean maps contract, and RETD's mean map is the slower one.  The sampled
sign reversal is caused by the random-product dynamics; the mean dynamics do
not explain it.  The difference comes from the final zero-ratio factor: ETD applies the identity after a
trace shock, whereas RETD applies $\bm D_s(c)$ and releases a delayed
correction.  Section~\ref{sec:preliminaries} supplies the exact positive-ETD
existence proof; this Baird analysis shows that the same mechanism persists in
the canonical multidimensional counterexample, with the stated distinction
between numerical ETD evidence and certified RETD contraction.

\FloatBarrier

\section{Experimental Study}
\label{sec:experiments}

The experiments evaluate four questions: \emph{Q1}, whether the new two-state
and Baird counterexamples show the predicted finite-horizon ETD expansion and
RETD contraction while other off-policy algorithms can learn on the same
trajectories; \emph{Q2}, how the RETD pathwise and mean-stability regions change
with leakage $c$, behavior mismatch, and discount; \emph{Q3}, whether the raw
RETD equilibrium shift can be removed from finite-sample outputs to recover
the ETD fixed prediction; and \emph{Q4}, how RETD behaves away from the
counterexamples on Boyan's chain and three random-walk feature representations.

For every learning curve in Q1, Q3, and Q4, the reported mean and sample
standard deviation are computed from $10{,}000$ independent runs.  Within an
environment and seed, all algorithms consume the same state--action
trajectory.  Hyperparameters for the contextual baselines are selected on ten
independent validation seeds and then frozen before the $10{,}000$ disjoint
test seeds are evaluated.  Saved seed-level outcomes, selected configurations,
aggregate curves, and checksums accompany the implementation.

Q2 estimates Lyapunov signs from long normalized random-matrix products over
parameter grids.  Its seed counts refer to independent long-product estimators
and are reported separately from the $10{,}000$ learning-curve replications.

\begin{figure}[!b]
\centering
\includegraphics[width=.98\linewidth]{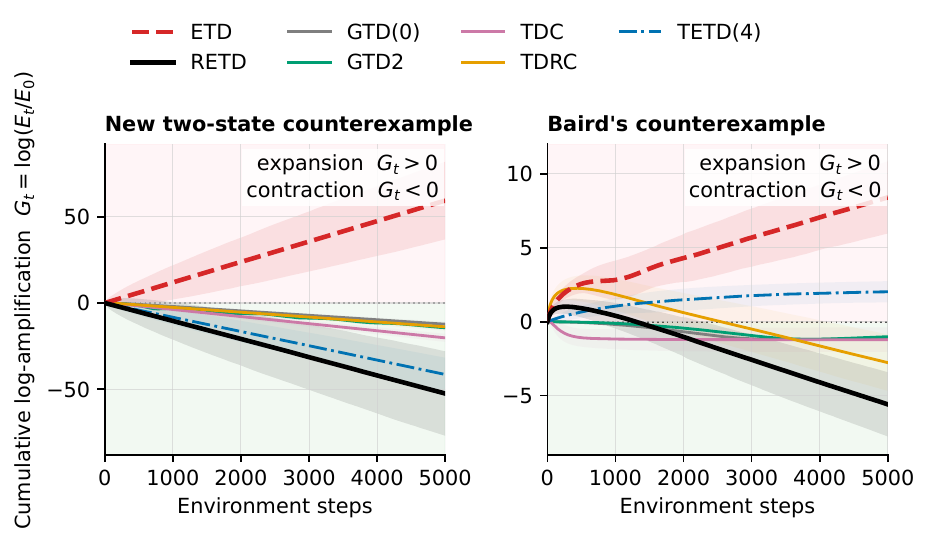}
\caption{Training curves on the two counterexamples over $10{,}000$
independent paired runs.  Curves are mean cumulative log-amplification $G_t$;
bands are one sample standard deviation.  The green half-plane $G_t<0$ means
contraction and the red half-plane $G_t>0$ means expansion.  Both axes are
linear.  ETD and RETD use the analyzed settings; the remaining methods use
configurations selected on independent validation trajectories.}
\label{fig:counterexample-baseline-curves-new}
\end{figure}

\subsection{Evaluation Metrics}

For the zero-reward counterexamples, let $E_t$ denote RMS value error over all
states.  We report $\widehat\lambda_t=t^{-1}\log(E_t/E_0)$ and
$G_t=t\widehat\lambda_t=\log(E_t/E_0)$.  Thus $G_t<0$ means contraction and
$G_t>0$ means expansion; more negative is
better.  Plotting $G_t$ preserves the sign and makes different contraction
rates visible on linear axes.  For the nonzero-reward tasks in Q3 and Q4, the
metric is RMS value error, for which lower is better.  Every shaded band is
one sample standard deviation across independent runs, and no temporal
smoothing or logarithmic axis is used.

\subsection{Q1: Finite-Horizon Counterexample Validation}

\emph{Settings and Sign Predictions.}

The new two-state experiment uses the construction in
Sections~\ref{sec:two-state-task} and~\ref{sec:two-state-retd}:
$p=.3$, $\gamma=.9$, $\phi(s_1)=1$, $\phi(s_2)=2$, zero reward, and
$\theta_0=1$.  The analyzed ETD point is $\alpha=.1$; RETD uses the same
stepsize, $c=4$, and $\omega_0=0$.  At $t=5{,}000$, the $10{,}000$ runs give
$\widehat\lambda_{5000}^{\rm ETD}=.011878\pm.004505$ and
$\widehat\lambda_{5000}^{\rm RETD}=-.010495\pm.004903$.
The exponent is positive in $99.90\%$ of ETD runs but only $2.23\%$ of RETD
runs.  This is the finite-horizon counterpart of the exact sign calculation,
not an outlier-driven difference.

For Baird's seven-state construction in Section~\ref{sec:baird}, ETD uses
$\alpha=.01$ and RETD uses $(\alpha,c)=(.01,.5)$.  The corresponding estimates
are $\widehat\lambda_{5000}^{\rm ETD}=.001682\pm.000488$ and
$\widehat\lambda_{5000}^{\rm RETD}=-.001118\pm.000435$,
with positive signs in $99.99\%$ and $1.84\%$ of runs, respectively.  The
negative RETD limiting sign is covered by the certificate in
Section~\ref{sec:baird}; the positive ETD sign remains a finite-horizon
numerical result.

\emph{Training-Curve Comparison with Off-Policy Baselines.}

Figure~\ref{fig:counterexample-baseline-curves-new} adds GTD(0), GTD2, TDC,
TDRC, and TETD(4) to the same paired test trajectories.  ETD and RETD retain
the settings fixed by the counterexample analyses.  The other methods use
configurations selected on independent validation trajectories; their update
scales differ, so forcing a common numerical stepsize would not be a fair
comparison.

At $t=5{,}000$ on the new two-state counterexample, mean $G_t$ is $59.39$ for
ETD, $-52.47$ for RETD, and $-41.43$ for the next-fastest TETD(4).  On Baird's
counterexample, the corresponding values are $8.41$ for ETD, $-5.59$ for
RETD, $-2.77$ for TDRC, and $2.02$ for TETD(4).  RETD is therefore the fastest
contracting method at the analyzed points in both environments.  The
contraction of the other methods shows that both tasks are learnable.  ETD's
positive curves arise from its sampled update product.

\subsection{Q2: Sensitivity to \texorpdfstring{$c$}{c}, Behavior Mismatch, and Discount}

\begin{figure}[!b]
\centering
\setlength{\abovecaptionskip}{4pt}
\includegraphics[width=.98\linewidth]{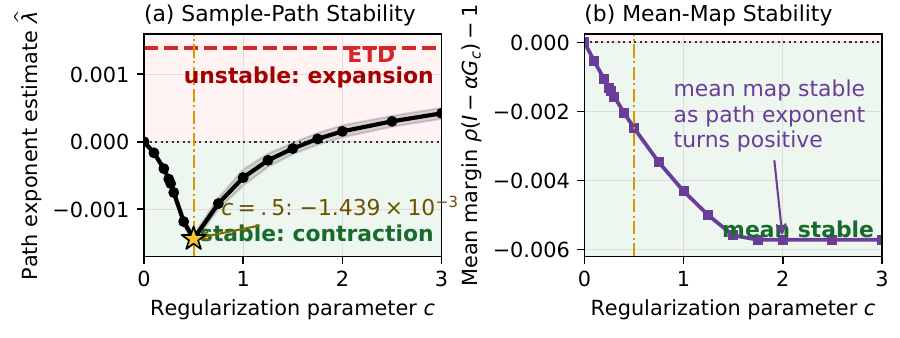}
\caption{Baird regularization path on linear axes.  Panel (a) reports the
sample-path exponent and panel (b) the deterministic mean-map margin;
negative values denote stability.}
\label{fig:baird-c-path-new}
\end{figure}

\emph{Why $c$ Cannot Be Chosen from Mean Stability Alone.}

The first scan holds Baird's environment and $\alpha=.01$ fixed.  The full numerical
scan evaluates 23 values from $c=0$ to $150$ using ten independent normalized
matrix-product estimates, each with 10,000 burn-in and 200,000 retained
matrices.  Figure~\ref{fig:baird-c-path-new} displays the 16 values in
$[0,3]$, where the contraction--expansion transition occurs; every evaluated
point with $c\ge4$ remains on the positive side of the pathwise boundary and
is retained in the released result file.  The band is one sample standard
deviation over the ten estimates, and the star marks $c=.5$, whose negative
sign also has the rigorous cycle certificate.

The upper panel falls from the neutral $c=0$ point to
$-.001439\pm.000061$ at $c=.5$, then crosses zero between $c=1.5$ and
$c=1.75$.  Hence neither zero leakage nor arbitrarily large leakage is
appropriate.  The lower panel stays below zero over the displayed range.  The
two diagnostics have different signs:
$\rho(I-\alpha G_c)<1\not\Rightarrow\lambda_{\rm top}<0$.  Thus $c$ changes
the sampled feedback dynamics even when the deterministic
mean matrix remains positive stable.

\FloatBarrier

\begin{figure}[!b]
\centering
\includegraphics[width=.98\linewidth]{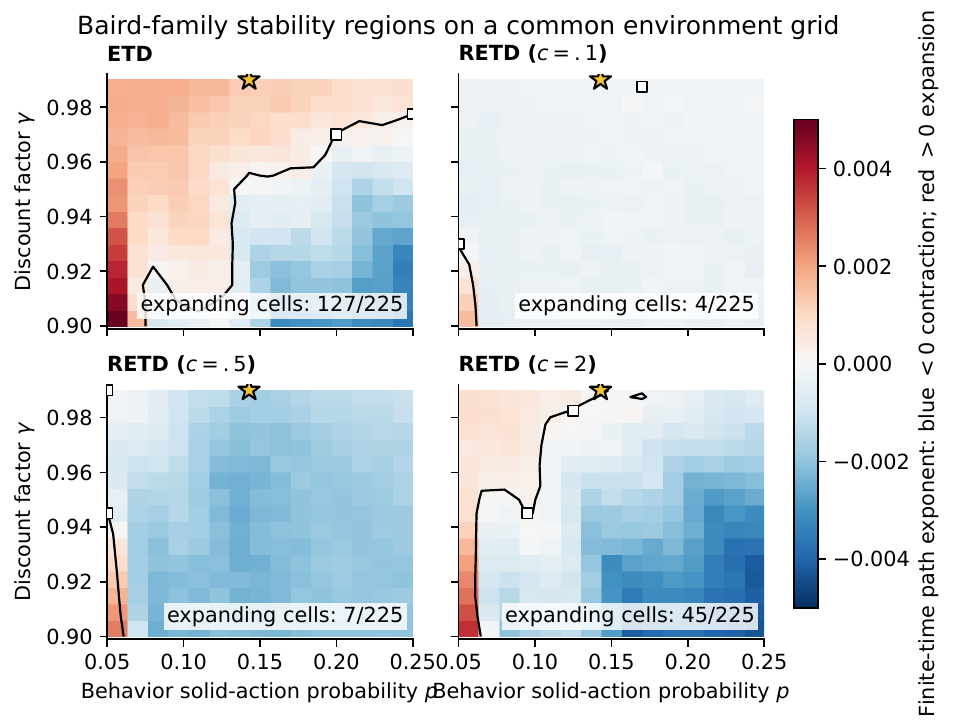}
\caption{Finite-time path-exponent map over one common Baird-family grid.
Blue cells have negative exponents and contract; red cells have positive
exponents and expand.  Black curves mark the zero boundary, stars mark the
canonical Baird point, and open squares are independent boundary checks.  The
in-panel count reports expanding cells out of 225.  All mean systems on this
grid are positive stable, so the color changes show changes in the sample-path
exponent.}
\label{fig:baird-phase-new}
\end{figure}

\emph{Stability over a Family of Baird Environments.}

The second scan varies the behavior solid-action probability $p$ and discount factor
$\gamma$ while holding the feature geometry, target policy, reward, reset
mechanism, and $\alpha=.01$ fixed.  The common $15\times15$ grid covers
$p\in[.05,.25]$ and $\gamma\in[.90,.99]$.  Each cell is screened by ten
independent long-product estimates; open squares identify boundary cells
rerun with 30 new long-product estimates.  These counts characterize a
random-matrix diagnostic, not repeated learning curves.

ETD expands in 127 of 225 cells.  RETD expands in 4, 7, and 45 cells at
$c=.1,.5,$ and $2$, respectively.  The increase from 7 expanding cells at
$c=.5$ to 45 at $c=2$ shows that the stable region is not monotone in $c$.
Only the canonical RETD$(.5)$ point is covered by the rigorous certificate;
the regional map is numerical and provides no uniform guarantee over the
grid.
\FloatBarrier

\subsection{Q3: Recovery of the ETD Fixed Prediction}
\label{sec:experiment-recovery-new}

Zero reward isolates homogeneous stability but cannot test the affine
solution identities.

\begin{figure}[!b]
\centering
\includegraphics[width=.98\linewidth]{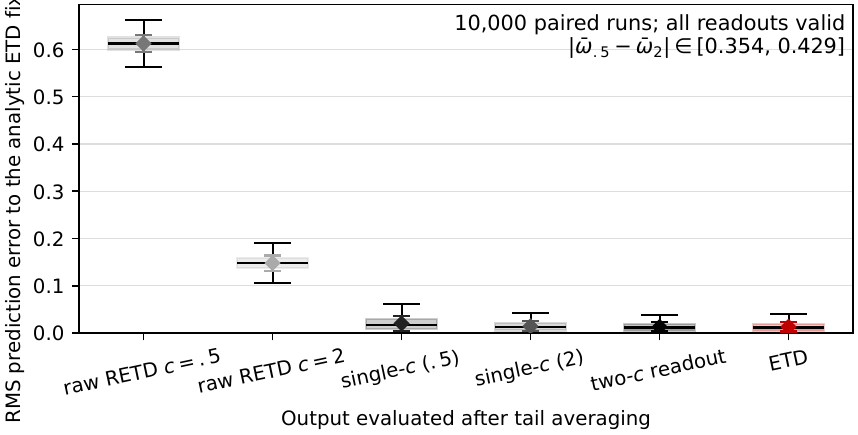}
\caption{Absolute prediction error after tail averaging over $10{,}000$
paired runs.  Points are means and bars are one sample standard deviation;
lower is better.  Raw RETD retains the predicted $c$-dependent shift, whereas
the single-$c$ and two-$c$ readouts return to the ETD error scale.  The
ordinate is linear.}
\label{fig:fixed-point-recovery-new}
\end{figure}

We use a separate bounded-trace, nonzero-reward
two-state task solely to evaluate fixed-point recovery.  Its transition
matrices are
\[
 P_{\rm left}=\begin{bmatrix}1&0\\1&0\end{bmatrix},
 \qquad
 P_{\rm right}=\begin{bmatrix}0&1\\0&1\end{bmatrix}.
\]
The behavior policy chooses right with probability $3/4$, the target always
chooses right, $\gamma=1/2$, and $\phi=(1,3/5)^\top$.  Reward is one precisely
when the current state is $s_2$ and action right is selected, and zero
otherwise.  Hence $d_\mu=(1/4,3/4)^\top$, $\rho_{\rm right}=4/3$,
$\rho_{\rm left}=0$, and $\gamma\rho_t\le2/3$, so the follow-on trace is
bounded.

Direct evaluation gives
\[
 \begin{aligned}
 A&=\frac{49}{100},\quad q=\frac{7}{10},\quad b=\frac{21}{20},
 \quad \beta=\frac74,\\
 \theta_E&=\frac{15}{7},\quad u=A^{-1}q=\frac{10}{7},
 \quad q^\top u=1,\quad m_E=\frac14.
 \end{aligned}
\]
The raw RETD equilibria are $(\theta,\omega)=(10/7,1/2)$ at $c=.5$ and
$(55/28,1/8)$ at $c=2$.  Thus the raw shift is nonzero and the ETD target is
known analytically.

We run $10{,}000$ paired trajectories for 20,000 transitions from zero, using
$\alpha_t=5/(t+100)$ and averaging the last 4,000 iterates.  The single-$c$
output applies $\theta_E=\theta+u\omega$; the two-$c$ output combines two RETD
learners using the identity in Section~\ref{sec:retd-fixed-point} and does not
use $u$.

The mean errors are $.61241\pm.01827$ and $.14789\pm.01608$ for raw RETD at
$c=.5$ and $c=2$.  Single-$c$ recovery reduces them to
$.02028\pm.01526$ and $.01435\pm.01082$; two-$c$ recovery gives
$.01305\pm.00984$, compared with $.01330\pm.01013$ for ETD.  No two-$c$
denominator is invalid: its absolute value ranges from $.3539$ to $.4293$.
The recovered errors agree with the ETD fixed-point identity.  Two-$c$ recovery
maintains two learners and becomes ill conditioned when their empirical
auxiliary states are too close.
\FloatBarrier

\subsection{Q4: Additional Prediction Tasks}

Q4 uses Boyan's chain and an off-policy random walk with tabular, inverted, and
linearly dependent features.  These tasks evaluate the dependence of
finite-sample performance on the task and feature representation outside the
two instability examples.

\begin{figure}[!b]
\centering
\includegraphics[width=.98\linewidth]{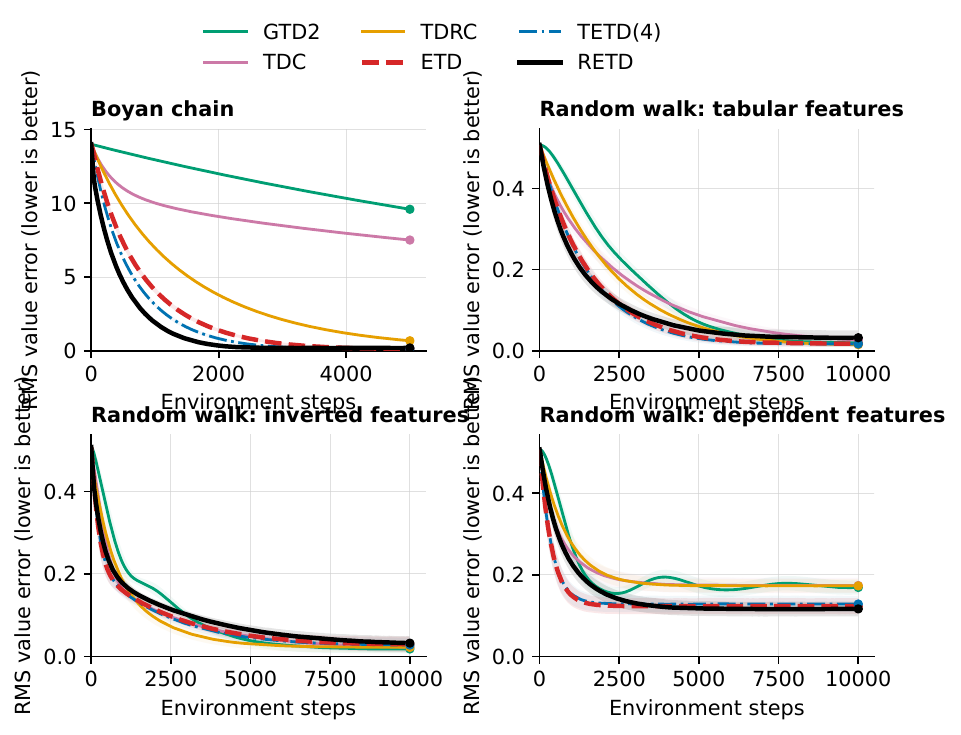}
\caption{Learning curves on Boyan's chain and three representations of the
off-policy random walk.  Curves are mean RMS value error over $10{,}000$
independent held-out runs; bands are one sample standard deviation.  Lower is
better, both axes are linear, and no temporal smoothing is applied.}
\label{fig:additional-prediction-features-new}
\end{figure}

The 13-state Boyan chain \citep{boyan1999least} has states
$0,\ldots,12$.  From any $s\ge2$, the behavior process moves to $s-2$ with
probability $.4$ and reward $1$, or to $s-1$ with probability $.6$ and reward
$2$; the target probabilities are $.5$ and $.5$, so the two importance ratios
are $1.25$ and $5/6$.  State $1$ moves deterministically to state $0$ with
reward $1.5$.  From state $0$, a new state is sampled uniformly from all 13
states with reward $1.5$ and ratio one.  We use $\gamma=.9$ and
$\theta_0=\bm 1_4$.  Let $e_j$ be the $j$th coordinate vector in
$\mathbb R^4$.  For $s=4j+r$, $j\in\{0,1,2\}$ and
$r\in\{0,1,2,3\}$, $\phi(s)=(1-r/4)e_{j+1}+(r/4)e_{j+2}$ and
$\phi(12)=e_4$.

The random walk has five nonterminal states $0,\ldots,4$, starts each episode
at state $2$, and moves left or right with behavior probabilities
$(.5,.5)$ and target probabilities $(.4,.6)$.  Thus the importance ratios are
$.8$ and $1.2$.  Crossing the left boundary gives reward zero, crossing the
 right boundary gives reward one, and either event restarts the episode at
 state $2$ with the follow-on trace reset to one.  Here $\gamma=.9$ and
$\theta_0=0$.  The three feature matrices, with states indexing rows, are
$\Phi_{\rm tab}=I_5$, $\Phi_{\rm inv}=\tfrac12(\bm1\bm1^\top-I_5)$, and
$\Phi_{\rm dep}=\left[\begin{smallmatrix}
1&0&0\\
1/\sqrt2&1/\sqrt2&0\\
1/\sqrt3&1/\sqrt3&1/\sqrt3\\
0&1/\sqrt2&1/\sqrt2\\
0&0&1
\end{smallmatrix}\right]$, respectively.
Boyan's chain is run for 5,000 transitions and each random-walk variant for
10,000 transitions.

We compare GTD2, TDC, TDRC, ETD, TETD(4), and RETD.  For each task, every
method receives 24 candidate configurations evaluated on the same ten
validation seeds.  The selected configuration is frozen and then evaluated
on $10{,}000$ new integer seeds.  All methods share the same counter-based
random stream within a task and seed.  RETD is evaluated through its analytic
single-$c$ recovered prediction, so any remaining difference is not the raw
$c$-dependent fixed-point shift.

The terminal RMS errors depend on the feature representation.  On Boyan's
chain, TETD(4) is best at $.13884\pm.04487$, followed by ETD at
$.15301\pm.05294$ and RETD at $.19145\pm.03998$.  On the tabular random walk,
TDRC is best at $.01690\pm.00568$, whereas ETD and RETD attain
$.01797\pm.00504$ and $.03239\pm.01199$.  With inverted features, GTD2 is best
at $.01859\pm.00633$; ETD and RETD attain $.03069\pm.00695$ and
$.03515\pm.00578$.  With linearly dependent features, in contrast, RETD is
best at $.11633\pm.01162$, compared with $.12337\pm.01100$ for ETD and
$.12868\pm.01117$ for TETD(4).  No method diverges on any of these four tasks.

These tasks evaluate RETD outside the two instability examples.  After
hyperparameter selection, its prediction error is lower on the linearly
dependent random walk but higher on Boyan's chain and the other two random-walk
representations.  The finite-sample advantage therefore depends on the task
and feature representation; the results do not establish uniform dominance
over the baselines.
\FloatBarrier

\section{Related Work}
\label{sec:related}

We group related methods by the object they modify: the projected equation,
the trace or ratios, the update dynamics, or the stability criterion.
Table~\ref{tab:related-methods} summarizes the distinctions most relevant to
RETD.

\begin{table}[tb]
\centering
\footnotesize
\setlength{\tabcolsep}{2pt}
\renewcommand{\arraystretch}{.96}
\begin{tabularx}{\linewidth}{>{\raggedright\arraybackslash}p{.22\linewidth}>{\raggedright\arraybackslash}p{.27\linewidth}X}
\toprule
Family & Modified object & Relation to RETD \\
\midrule
GTD/GTD2/TDC/\\TDRC & Gradient correction with an auxiliary vector &
Targets an ordinary projected or MSPBE-related system; it is not an emphatic
shock-state modification and does not use the RETD recovery identity \\
Generalized, PER, learned ETD, and TETD & Emphasis or trace pathway &
Changes how emphasis is weighted, estimated, or truncated; RETD leaves $F_t$
and $\rho_t$ unchanged \\
V-trace and ratio clipping & Ratios or the effective operator &
Changes the sampled ratios or target operator, whereas RETD changes the joint
post-shock product \\
RETD & Scalar leaky post-shock state &
Affine-shifted system with exact ETD recovery, diminishing-stepsize convergence,
and conditional constant-stepsize path stability \\
\bottomrule
\end{tabularx}
\caption{Related off-policy methods grouped by the component they modify.}
\label{tab:related-methods}
\end{table}

\emph{The First Two Off-Policy Problems.}

Baird's counterexample established that ordinary linear off-policy TD can
diverge even with bounded features and zero reward
\citep{baird1995residual}.  Gradient-TD and TDRC stabilize projected or
gradient-correction systems \citep{sutton2008convergent,sutton2009fast,
ghiassian2020gradient}; RETD instead uses a scalar $\omega_t$ to store and
leak an emphatic TD shock.  It is therefore not a scalar GTD2/TDC variant: its
recovered equilibrium is the ETD solution, whereas those methods target an
ordinary MSPBE-related system.  Oblique-projection work further shows why this
recovery matters: a stable projected recursion can have a poor target-policy
approximation \citep{scherrer2010should,kolter2011fixed}, while emphatic
weighting changes that projection geometry
\citep{sutton2016emphatic,hallak2016generalized}.

\emph{Modifications of Emphatic Traces.}

The original ETD study identified high trace variance
\citep{sutton2016emphatic}; later work established diminishing-stepsize
convergence and constrained constant-step weak convergence
\citep{yu2015convergence,yu2016weak}.  These results concern different
stability objects from an unconstrained homogeneous constant-stepsize product.

PER-ETD, TETD, learned or deep emphatic methods, LC-ETD, emphatic
actor--critic methods, and V-trace respectively estimate, truncate, relax, or
replace parts of the emphasis or ratio pathway
\citep{guan2022per,zhang2022truncated,jiang2021emphatic,jiang2022learning,
he2023loosely,graves2023offpolicy,espeholt2018impala}.

RETD occupies a different point: it leaves
$F_{t+1}=1+\gamma\rho_tF_t$ and $\rho_t$ unchanged, hence does not change the
trace tail index.  The same shock instead enters $\theta_t$ and $\omega_t$;
later zero-ratio transitions use the leaked scalar for recovery.  This is the
algebraic distinction in Equation~\eqref{eq:etd-retd-matrices-new} and the
dynamical distinction in the exact two-state cycle products.

\emph{Random Products and Lyapunov Stability.}

Top Lyapunov exponents and regenerative drift arguments are standard tools for
random products and Markov chains
\citep{furstenberg1960products,kingman1968ergodic,oseledets1968multiplicative,
bougerol1985products,meyn2009markov}.

The external small-stepsize random-product result used in
Theorem~\ref{thm:retd-fixed-new} is due to \citet{durmus2021stability}.  RETD
requires an algorithm-specific verification: its sample matrix contains the
rank-one $c$-term in Equation~\eqref{eq:sample-c-rank-one-new}, the explicit
condition on $c$ makes the mean matrix positive stable, and the resulting
Lyapunov matrix and growth constant determine an admissible stepsize.  Recent
TD analyses also use exponential stability
\citep{samsonov2024exponential,lee2026single}, but they do not verify the
augmented trace-driven RETD matrices or compare their post-shock dynamics with
ETD.  Satisfying the RETD condition therefore proves no ETD result.

The Baird certificate in Section~\ref{sec:baird} combines these tools with an
environment-specific projective drift computation and separates fitted
quantities from rigorous interval verification \citep{rump2010verification}.
The RETD point has a projective certificate; the positive Baird ETD exponent is
only a finite-horizon numerical estimate.  RETD also differs from optimization
analyses of multiplicative noise and heavy tails
\citep{simsekli2019tail,gurbuzbalaban2021heavy,hodgkinson2021multiplicative,
chemnitz2025stability}: its matrices are trace-driven and prediction-relevant
exponents require removal of feature-null directions.

\emph{Relation to Bellman-Error Centering.}

Bellman-error centering stores a shared TD-error component
\citep{chen2026bellman}, but RETD drives its scalar state by
$F_t\rho_t\delta_t$, couples it to an emphatically weighted update, and targets
a recoverable ETD fixed point.  Its leaky impulse response, $c=0$ conserved-mode
obstruction, two-state sign separation, and Baird certificate are therefore
distinct from centered TD's objectives and stability mechanism.

\section{Conclusion and Future Work}
\label{sec:conclusion}

This paper separates mean stability, diminishing-stepsize convergence, and
constant-stepsize sample-path stability in off-policy emphatic prediction.  A
two-state counterexample proves that ETD can have a strictly
contractive mean map and a positive top Lyapunov exponent at the same
stepsize.  RETD adds one leaky scalar state, reverses the product sign on that
counterexample, and preserves the ETD solution through an exact affine
readout.  The theory establishes almost-sure convergence with a diminishing
stepsize and conditional moment contraction with a sufficiently small constant
stepsize; the Baird study and held-out experiments show both the resulting
stability gains and their task dependence.

Future work will extend the analysis to RETD($\lambda$).  A multi-step
extension must treat
the eligibility trace, follow-on trace, parameter vector, and recovery state as
one coupled stochastic system; derive the corresponding mean matrix and
fixed-point recovery identity; and establish diminishing- and constant-
stepsize results without assuming the random-product contraction that the
algorithm is intended to prove.

\clearpage

\acks{This study was supported in part by the National Natural Science
Foundation of China (Nos. 62276142, 62206133, 62202240, and 62506172),
the Fundamental Research Funds for the Central Universities
(No. 2242025K30024), the Basic Research Program of Jiangsu
(No. BK20250658), and the Natural Science Research Start-up Foundation
of Recruiting Talents of Nanjing University of Posts and Telecommunications
(No. NY225026).

\emph{Competing Interests.} The authors declare that they have no competing
interests relevant to this work.}

\bibliography{references}

\end{document}